\documentclass{article}

\PassOptionsToPackage{numbers,compress}{natbib}
\usepackage[preprint]{neurips_2026}

\usepackage[utf8]{inputenc} % allow utf-8 input
\usepackage[T1]{fontenc}    % use 8-bit T1 fonts
\usepackage{hyperref}       % hyperlinks
\usepackage{url}            % simple URL typesetting
\usepackage{booktabs}       % professional-quality tables
\usepackage{amsfonts}       % blackboard math symbols
\usepackage{nicefrac}       % compact symbols for 1/2, etc.
\usepackage{microtype}      % microtypography
\usepackage{xcolor}         % colors

\usepackage{threeparttable}
\usepackage{amsmath,amssymb,amsthm}
\usepackage{graphicx}
\usepackage{accents}
\usepackage{bm}
\usepackage{tcolorbox}
\usepackage{algorithm}
\usepackage{algpseudocode}
\algnotext{EndFor}
\usepackage{enumitem}
\usepackage{float}

\newtheorem{proposition}{Proposition}
\newtheorem{lemma}{Lemma}

\title{Large Language Models for Sequential Decision-Making: Improving In-Context Learning via Supervised Fine-Tuning}

\author{%
  Minmin Zhang \\
  Harvard University \\
  \texttt{minminzhang@hks.harvard.edu}
  \And
  Sina Aghaei \\
  Harvard University \\
  \texttt{sina.aghaei1994@gmail.com}
  \And
  Soroush Saghafian \\
  Harvard University \\
  \texttt{soroush\_saghafian@hks.harvard.edu}
}

\begin{document}

\maketitle

\begin{abstract}
Large language models (LLMs) have shown remarkable in-context learning (ICL) capabilities, yet their potential for sequential decision-making remains underexplored. In this paper, we study the ICL capabilities of LLMs in sequential decision-making settings, including Markov Decision Processes (MDPs), Partially Observable MDPs (POMDPs), and Ambiguous POMDPs (APOMDPs). We fine-tune pretrained LLMs to perform few-shot decision-making directly from offline, oracle-labeled trajectories. Our framework enables flexible imitation of policies through supervised fine-tuning (SFT). Theoretically, we focus on linear MDPs and interpret a fine-tuned attention layer as implicitly estimating optimal Q-functions from in-context data. Building on this interpretation, we derive an end-to-end suboptimality bound for the induced policy that separates the in-context estimation error from the training-length bias. Empirically, across synthetic MDP, POMDP, and APOMDP settings, we find that fine-tuned LLMs achieve substantially smaller optimality gaps than in-context-only and random baselines, with especially large gains in longer-horizon, partially observed, and model-ambiguous environments. Together, these results show that supervised fine-tuning provides an effective route to endowing pretrained LLMs with sequential decision-making capabilities from offline data, which is an important advantage in domains such as healthcare where offline data are abundant.

\end{abstract}

%%%%%%%%%%%%%%%%%%%%%%%%%%%%%%%%%%%%%%%%%%%%%%%%%%%%%%
\section{Introduction}\label{sec:intro}
%%%%%%%%%%%%%%%%%%%%%%%%%%%%%%%%%%%%%%%%%%%%%%%%%%%%%%

Sequential decision-making under uncertainty is central to many high-stakes domains, from clinical treatment planning to resource management. In these settings, an agent must choose a sequence of actions over time to maximize cumulative rewards, using only observational data collected under some historical policy. Dynamic Treatment Regimes (DTRs) provide a principled framework for discovering effective policies from such data~\cite{robins1986new, murphy2003optimal}. However, standard DTR approaches rely on the strong assumption that all confounders are observable or that their effects can be ignored (sequential ignorability), which is often violated in practice. To address this challenge, Ambiguous Dynamic Treatment Regimes (ADTRs) account for unobservable confounders by evaluating policies against a set of plausible data-generating models (model ambiguity) rather than a single one~\cite{saghafian2018ambiguous, saghafian2023ambiguous, saghafian2025insight}.
Furthermore, previous studies such as~\cite{saghafian2023ambiguous} connect ADTRs to APOMDPs and develop reinforcement learning (RL) approaches to find the best sequence of actions. Compared with MDPs or POMDPs, APOMDPs provide a more realistic framework for learning sequential decision-making from observational data in the presence of unobservable confounders by directly considering the underlying model ambiguity.

Large language models have shown remarkable in-context learning capabilities. Given a small set of demonstrations as the input prompt, a pretrained LLM can make predictions on new instances without updating its parameters~\cite{brown2020language}. This ability has been studied in supervised learning~\cite{garg2022can, zhang2023trained} and, more recently, in sequential decision-making settings~\cite{DT2021chen, AD2022laskin, DPT2023lee}. However, prior work on in-context learning for sequential decision-making mainly trains transformer architectures from scratch, focuses on fully observable MDPs, and does not consider the real-world challenges caused by partial observability or model ambiguity.

In this study, we adapt supervised fine-tuning of pretrained LLMs to sequential decision-making in MDPs, POMDPs, and APOMDPs. Rather than training a transformer from scratch, we fine-tune an open-source pretrained LLM so that it can first read a few demonstration trajectories and then act in new, unseen tasks without any further parameter updates. Building on pretrained LLMs leverages the broad representational capacity acquired during pretraining, which may improve generalization across diverse tasks. Starting from pretrained LLMs also provides an avenue for integrating diverse inputs and human intuition inclusion in real-world sequential decision-making applications, augmenting the capabilities of human-AI centaurs~\cite{saghafian2025insight,saghafian2024effective}. Our main contributions in this work are as follows.

\begin{itemize}[leftmargin=*]
    \item \textbf{SFT framework for in-context sequential decision-making.} We fine-tune a pretrained LLM on offline trajectories so that it can make sequential decisions in new tasks given only a few demonstrations at test time. The framework uses parameter-efficient adapters and trains on serialized trajectories. At test time, this approach allows the model to read few-shot demonstrations and act in new tasks without any further parameter updates.

    \item \textbf{Offline data use with oracle supervision.} Our framework leverages offline datasets along with an optimal (or near-optimal) action oracle to synthesize high-quality trajectories for fine-tuning. This practice is particularly appealing in domains such as healthcare, where experimentation is costly (or unethical) and model-based or expert-derived supervision can often be constructed from logged observational data.
    
    \item \textbf{Theoretical analysis in linear MDPs.} Adopting the in-context linear-regression framework of \cite{zhang2023trained}, we interpret a trained linear self-attention layer as implicitly estimating the parameters of the optimal Q-function for a new task. Our theoretical contribution (a) makes use of this prediction-level interpretation to create an end-to-end suboptimality bound for the induced policy and (b) decomposes the bound into an in-context estimation error (decreasing in the number of support trajectories) and a training-length bias (decreasing in the number of training trajectories).

    \item \textbf{Empirical evaluation across MDPs, POMDPs, and APOMDPs.} We show that fine-tuned LLMs achieve low optimality gaps and consistently outperform both random policies and ICL-only baselines. Specifically, we find that fine-tuning roughly halves the optimality gap in long-horizon MDPs. In POMDPs and APOMDPs, fine-tuned LLMs achieve large gains even under partial observability and model ambiguity, suggesting robust performance for our approach across different settings.

\end{itemize}

\section{Related work}

\paragraph{In-context prediction.} Existing research has investigated the ICL capabilities of LLMs in classification and regression tasks and shown that transformer-based architectures can learn general function classes when trained on in-context prompts~\cite{garg2022can, zhang2023trained, mirchandani2023large, zhao2024probing}. Some studies (see, e.g.,~\cite{zhang2023trained}) provide theoretical guarantees for learning linear regression in-context. LLMs have also been applied to time series analysis (see, e.g., \cite{zhang2024large} for a detailed survey).

\paragraph{LLMs as decision-making agents.} Some studies have explored using LLMs directly for decision-making through prompting. \cite{yao2023react} generates both reasoning traces and task-specific actions in an interleaved manner, and \cite{huang2022inner} uses language feedback for embodied planning. \cite{reed2022generalist} trains a single generalist transformer across vision, language, and control tasks. These approaches either use pretrained models directly or train models from scratch. In contrast, we fine-tune pretrained LLMs for few-shot sequential decision-making.

\paragraph{In-context decision-making.} Decision Transformer was among the first to cast RL as sequence modeling, training a transformer to predict actions conditioned on past states, actions, and a target return~\cite{DT2021chen}. Algorithm Distillation instead trains transformers to imitate the behavior of RL algorithms from their learning histories~\cite{AD2022laskin}. Decision-Pretrained Transformer (DPT) builds on these ideas and trains a transformer to imitate an optimal action oracle given in-context data~\cite{DPT2023lee}. Pretrained LLM representations are also used as input to added layers for decision-making under partial observability~\cite{li2022pre}, with partial observability referring to limited visual input rather than noisy observations of latent states. A survey of related studies in this stream can be found in~\cite{cao2024survey}.

Our study differs from these approaches in two key aspects. First, we fine-tune a pretrained LLM rather than training a transformer from scratch, preserving the representational capacity acquired during pretraining. Second, we extend the evaluation from MDPs to POMDPs and APOMDPs, more general settings with partial observability and model ambiguity that have not been fully studied in the ICL decision-making literature and are more appropriate for many applications. We provide more discussion of related studies in RL, causal inference, and LLMs in Appendix~\ref{sec:appendix_additional_related_work}.

% %%%%%%%%%%%%%%%%%%%%%%%%%%%%%%%%%%%%%%%%%%%%%%%%%%%%%%
\section{Problem setup and LLMs}\label{sec:problem_setup_and_llms}
% %%%%%%%%%%%%%%%%%%%%%%%%%%%%%%%%%%%%%%%%%%%%%%%%%%%%%%

We first present a general framework for sequential decision-making that covers MDPs, POMDPs, and APOMDPs as special cases. We then describe two approaches to preparing LLMs for sequential decision-making problems.

\subsection{Sequential decision-making}

We consider a setup in which an agent makes a sequence of decisions over a finite horizon $T_\tau$ for each task $\tau$ drawn from a distribution $\mathcal{T}$.
A task $\tau$ may vary across draws in its transition dynamics, reward function, and state, action, and reward spaces.
For a fixed task $\tau$, let $\mathcal{X}_{\tau}$, $\mathcal{A}_{\tau}$, and $\mathcal{R}_{\tau}$ denote the observation, action, and reward spaces, respectively.

At the beginning of each period $t \in \{1, \ldots, T_\tau\}$, the agent first observes $X_t \in \mathcal{X}_{\tau}$, then chooses an action $a_t \in \mathcal{A}_{\tau}$, after which a reward $r_t \in \mathcal{R}_{\tau}$ is realized. The observable history before action $a_t$ is $H_t \triangleq (X_1, a_1, r_1, \ldots, X_{t-1}, a_{t-1}, r_{t-1}, X_t)$.
A policy is a sequence $\pi_\tau = (\pi_{\tau,t})_{t=1}^{T_\tau}$, where $\pi_{\tau,t}(\cdot \mid H_t) \in \Delta(\mathcal{A}_{\tau})$ maps the observable history to a distribution over actions and we denote by $\Delta(\mathcal{A}_{\tau})$ the probability simplex over the action space $\mathcal{A}_{\tau}$.
Given a discount factor $\gamma \in (0,1]$, the performance is $J_{\tau}(\pi_\tau) \triangleq \mathbb{E}[\sum_{t=1}^{T_\tau} \gamma^{t-1} r_t ]$. The goal is to find a policy that maximizes this performance. In what follows, we suppress the subscript $\tau$ when the task is clear from context.

\subsection{Special cases}\label{sec:special_cases}

We discuss three special cases under the assumption that transition dynamics are Markovian. We provide the formal definitions and detailed notation for these cases in Appendix~\ref{sec:appendix_standard_definitions}.

\paragraph{MDP [full observability].}
In an MDP task $\mathcal{E}^{\mathrm{MDP}}_\tau = \langle \mathcal{S}_\tau, \mathcal{A}_\tau, P_\tau, R_\tau, \rho_\tau, T_\tau, \gamma \rangle$, the state $S_t \in \mathcal{S}_\tau$ is fully observable (i.e., $X_t = S_t$). We denote by $P_\tau(\cdot \mid s,a)$ the state-transition kernel, $R_\tau(s,a)$ the one-step reward, and $\rho_\tau$ the initial-state distribution. The Markov property allows restriction to policies $\pi_t(a \mid S_t)$.

\paragraph{POMDP [partial observability].}
When we observe only $O_t \in \mathcal{O}_\tau$ generated from a latent state $S_t \in \mathcal{S}_\tau$ via an observation kernel $Q_\tau(o \mid s, a)$, we set $X_t = O_t$. We maintain a belief vector $b_t \in \Delta(\mathcal{S}_\tau)$ updated via a Bayes operator $B_{\tau}$. The belief vector serves as a sufficient statistic for decision-making, which means we can restrict attention to belief-based policies $\pi_t(a \mid b_t)$.

\paragraph{APOMDP [model ambiguity].}
Using a single transition or observation kernel is often unreliable for real-world data. 
APOMDPs address this problem by equipping the task with an ambiguity set $\mathcal{M}_\tau$ of plausible models $m$, each specifying kernels $P_\tau^m$ and $Q_\tau^m$:
\[
\mathcal{E}^{\mathrm{APOMDP}}_\tau
=\big\langle \mathcal{S}_\tau,\; \mathcal{A}_\tau,\; \mathcal{O}_\tau,\; \mathcal{M}_\tau,\; P_\tau^m,\; Q_\tau^m,\; R_\tau,\; \rho_\tau,\; T_\tau,\; \gamma,\; \alpha \big\rangle,
\]
where $\alpha$ tunes pessimism~\cite{saghafian2018ambiguous, saghafian2023ambiguous}. One can derive an optimal policy that hedges against model ambiguity by solving the Bellman equation $V_t(b_t, \alpha) = \max_{a \in \mathcal{A}_\tau} \{U_t(b_t, \alpha, a) \}$, where
\[
U_t(b_t, \alpha, a) = b_t r_t + \alpha \gamma \min_{m \in \mathcal{M}_\tau} \{\mathcal{H}_{t+1}(b_t,\alpha, a, m)\} + (1-\alpha) \gamma \max_{m \in \mathcal{M}_\tau} \{\mathcal{H}_{t+1}(b_t,\alpha, a, m)\},
\]
and $\mathcal{H}_{t+1}(b,\alpha, a, m) \triangleq \sum_{o \in \mathcal{O}_\tau} \mathbb{P}(o | b,a,m)V_{t+1}(B_{\tau}(b, a, o,m), \alpha)$~\cite{saghafian2018ambiguous}. Setting $\alpha=1$ yields the pessimistic maximin criterion and $\alpha=0$ the optimistic maximax~\cite{saghafian2018ambiguous,saghafian2023ambiguous}.
APOMDPs generalize both MDPs and POMDPs by incorporating model ambiguity. When $\mathcal{M}_\tau$ is a singleton, MDPs arise under full observability and POMDPs arise under partial observability.

\subsection{LLMs for sequential decision-making}

We formulate sequential decision-making as an optimization problem in the settings mentioned above. An alternative framework is to recast the problem as a sequence-generation problem. Instead of directly optimizing the performance, an LLM learns a predictor of actions given the history $H_t$. Training proceeds over trajectories from diverse tasks $\tau$, and at test time the LLM generates actions in new tasks. This framework connects RL-style optimization with generative modeling, in which we view the policy as an inference mechanism that generates optimal actions conditioned on observations. Figure~\ref{figure:flowdiagram} illustrates the overall framework. In what follows, we consider two approaches, in-context learning and supervised fine-tuning, to prepare LLMs for sequential decision-making.

\begin{figure}[thbp]
\centering
\includegraphics[width=0.8\textwidth]{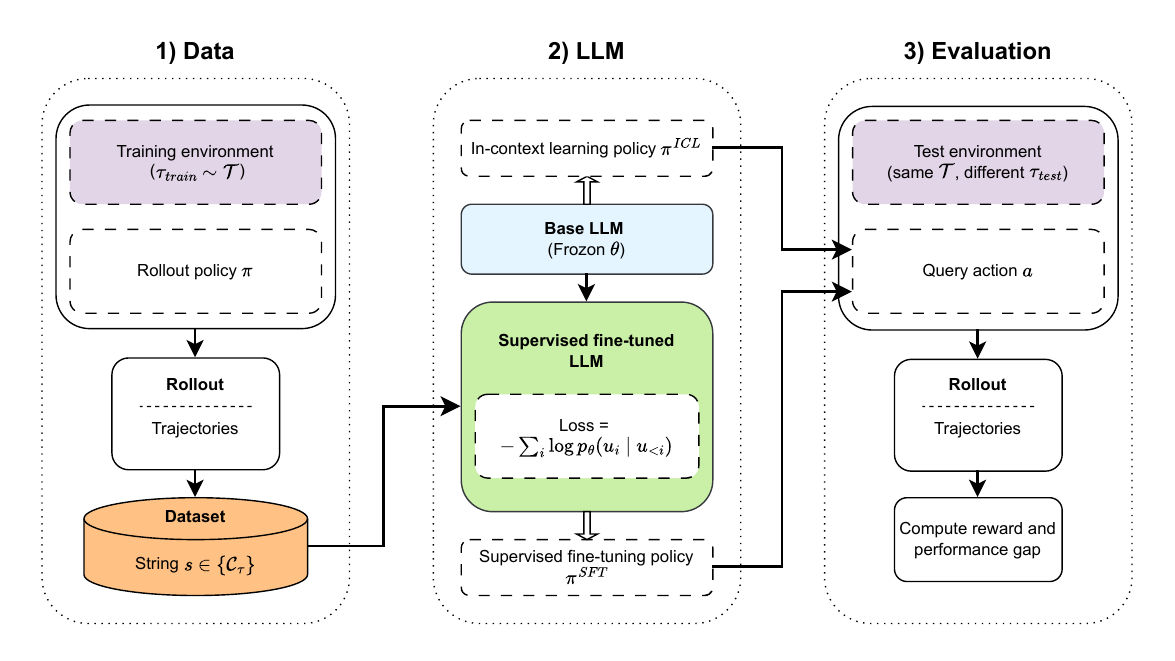}
\caption{Overview of data construction, policy generation, and performance evaluation.}
\label{figure:flowdiagram}
\end{figure}

\paragraph{In-context learning.}
ICL treats an LLM as a policy generator that adapts at inference time by conditioning on demonstrations in the prompt. We first draw a task $\tau_{\mathrm{test}}$ from $\mathcal{T}$. We then compute its optimal policy and roll out multiple few-shot support trajectories $\mathcal{D}_{\tau}$. Next, we serialize those trajectories into a textual context $\mathrm{ctx}_{\tau} = \mathsf{E}(\mathcal{D}_{\tau})$ via a deterministic encoder $\mathsf{E}$.
Finally, the LLM receives $(H_t, \mathrm{ctx}_{\tau})$ at the beginning of each period $t$ and generates an action $a_t \sim \pi^{\mathrm{ICL}}(\cdot \mid H_t, \mathrm{ctx}_{\tau})$, where we denote by $\pi^{\mathrm{ICL}} \in \Delta(\mathcal{A}_{\tau})$ the in-context learning policy.
In particular, we restrict outputs to action tokens and select the one with the highest probability.

\paragraph{Supervised fine-tuning.}
Fine-tuning adapts the pretrained parameters of the LLM to the target task distribution. We adopt parameter-efficient quantized low-rank adapters (QLoRA)~\cite{dettmers2024qlora}, which enable learning a low-rank update for each linear module while keeping the base weights frozen.
For each training task $\tau_{\mathrm{train}}$, we roll out multiple demonstration trajectories $\mathcal{C}_{\tau}$ under the optimal policy. We then serialize each trajectory into a string $s$. The training set comprises trajectories from multiple tasks drawn from $\mathcal{T}$.  After tokenization $s \mapsto (u_1,\ldots,u_L)$, SFT proceeds by minimizing the following loss function to obtain the optimal parameters $\theta^*$:
\begin{equation}\label{eq:sft_loss}
\mathcal{L}_{\mathrm{SFT}}(\theta)
\;=\;
\;\mathbb{E}_{\,s \in \{\mathcal{C}_{\tau}\}}
\Bigg[
-\sum_{i=1}^{L} \log p_\theta\!\big(u_i \,\big|\, u_{<i}\big)
\Bigg].
\end{equation}
This loss function matches the inference-time format. That is, the LLM learns to read few-shot demonstrations and continue the query trajectory in the same schema. We provide more details of fine-tuning in Appendix~\ref{sec:appendix_qlora}.
At evaluation, we use the same few-shot support structure as in the ICL setting, but now the policy is $\pi^{\mathrm{SFT}}(\cdot \mid H_t, \mathrm{ctx}_{\tau})$.
We provide more detailed pseudocode in Appendix~\ref{sec:appendix_pseudocode}.

%%%%%%%%%%%%%%%%%%%%%%%%%%%%%%%%%%%%%%%%%%%%%%%%%%%%%%
\section{Theoretical analysis}
%%%%%%%%%%%%%%%%%%%%%%%%%%%%%%%%%%%%%%%%%%%%%%%%%%%%%%
\label{sec:theoretical_analysis}

We develop a theoretical framework to understand why SFT improves in-context sequential decision-making. For tractability, our theoretical analysis considers a single linear self-attention (LSA) layer trained by gradient flow to predict optimal Q-values in linear MDPs. We build on the in-context linear-regression analysis of \cite{zhang2023trained}, which enables us to translate prediction-level guarantees into policy-level statements. We provide an end-to-end bound on the suboptimality of the policy induced by the LSA predictions, with an associated sample-complexity result (deferred to the appendix) relating the number of few-shot support trajectories and the training trajectory length to performance.

\subsection{Setup: In-context Q-value prediction in linear MDPs}
\label{subsec:theory_setup}

We consider a linear MDP task $\tau=\langle \mathcal S,\mathcal A,P,R,\rho,\gamma,\phi\rangle$ with known feature map $\phi:\mathcal S\times\mathcal A\to\mathbb R^d$. For a fixed task $\tau$, the optimal action-value function is assumed to be linear:
\begin{equation}
Q_\tau^*(s,a)=\langle \phi(s,a),w_\tau^*\rangle,
\qquad
w_\tau^*\in\mathbb R^d.
\label{eq:linear_q}
\end{equation}
For finite-horizon tasks, the time index can be absorbed into the feature map, e.g., by replacing $\phi(s,a)$ with $\phi_t(s,a)$ or $\phi(s,a,t)$. We suppress this notationally and write $\phi(s,a)$ throughout. We assume $w_\tau^* \stackrel{\mathrm{i.i.d.}}{\sim} \mathcal N(\mathbf{0}_d,\,I_d)$ after a standard rescaling of the features, where $\mathbf{0}_d\in\mathbb R^d$ denotes the zero vector and $I_d\in\mathbb R^{d\times d}$ the identity matrix.

\paragraph{Fixed-covariance reduction to in-context linear regression.}
Now let us fix a task $\tau$. We denote by $\mu$ the state-action distribution used to form the in-context examples, and define the feature covariance $\Lambda
\triangleq
\mathbb E_{(s,a)\sim\mu}\!\left[\phi(s,a)\phi(s,a)^\top\right]
\in\mathbb R^{d\times d}$.
Throughout the formal analysis we assume $\Lambda\succ 0$ is fixed across meta-training and test prompts, matching the fixed-covariance regime of \cite{zhang2023trained}. This holds, e.g., when the support trajectories come from a task-independent exploratory policy, or when the task distribution is restricted so that all tasks induce the same feature covariance. If $\Lambda$ varies, \cite{zhang2023trained} show that a single LSA layer can become biased under such covariate shifts. We therefore treat fixed $\Lambda$ as an explicit modeling assumption.

Let $x_i \triangleq \phi(s_i,a_i)$ and $y_i \triangleq Q_\tau^*(s_i,a_i)=\langle x_i, w_\tau^* \rangle$.
Following \cite{zhang2023trained}, we approximate the feature distribution by the matched Gaussian model
$x_i \stackrel{\mathrm{i.i.d.}}{\sim} \mathcal N(\mathbf{0}_d,\Lambda)$.
The $K_{\mathrm{test}}$ few-shot support trajectories each have length $T$. After the regression reduction, this gives $M \triangleq K_{\mathrm{test}}T$ in-context pairs $(x_i,y_i)_{i=1}^M$, together with a query feature $x_{\mathrm{q}}=\phi(s_{\mathrm{q}},a_{\mathrm{q}})$. Training trajectories contain $N \triangleq KT$ pairs from $K$ training trajectories.

\paragraph{LSA model and population training.}
Given $P=(x_1,y_1,\ldots,x_M,y_M,x_{\mathrm q})$, we define the embedding matrix $E=\bigl(\begin{smallmatrix} x_1 & \cdots & x_M & x_{\mathrm q}\\ y_1 & \cdots & y_M & 0 \end{smallmatrix}\bigr)\in\mathbb R^{(d+1)\times(M+1)}$.
The LSA layer with parameters $\theta=(W_{KQ},W_{PV})$ computes $f_{\mathrm{LSA}}(E;\theta) = E + W_{PV}E\,E^\top W_{KQ}E/M$.
The predicted Q-value $\widehat Q(x_{\mathrm q};\theta)$ is the bottom-right entry of $f_{\mathrm{LSA}}(E;\theta)$. The LSA layer is trained by gradient flow on the population mean-squared error
\begin{equation}
L(\theta)
=
\frac12
\mathbb E_{w,x_i,x_{\mathrm q}}
\left[
\left(
\widehat Q(x_{\mathrm q};\theta)
-
\langle w,x_{\mathrm q}\rangle
\right)^2
\right],
\label{eq:population_lsa_loss}
\end{equation}
where $w\sim\mathcal N(\mathbf{0}_d, I_d)$ and $x_i,x_{\mathrm q}\stackrel{\mathrm{i.i.d.}}{\sim}\mathcal N(\mathbf{0}_d,\Lambda)$. All formal results below are stated for this population objective. With finitely many training tasks, additional statistical and optimization errors can be added to the prediction error and we do not assign a specific rate to that finite-sample term here.

\subsection{In-context Q-prediction error}
\label{subsec:q_prediction_error}

The first ingredient is a per-query Q-prediction guarantee for the trained LSA layer. Define
\begin{equation}
\Gamma
\triangleq
\left(1+\frac1N\right)\Lambda
+
\frac{\operatorname{tr}(\Lambda)}{N}I_d .
\label{eq:gamma_def}
\end{equation}
The following lemma is a direct restatement of Theorems 4.1--4.2 and Corollary 4.3 of \cite{zhang2023trained} in our linear MDP notation.

\begin{lemma}[In-context Q-prediction error]
\label{lem:q_prediction}
Suppose the initialization satisfies Assumption 3.3 of \cite{zhang2023trained}, with initialization scale small enough for their convergence theorem. Gradient flow on \eqref{eq:population_lsa_loss} converges to a global minimizer $\theta^*$. At this minimizer, for any $P=(x_1,y_1,\ldots,x_M,y_M,x_{\mathrm q})$ drawn from an unseen task with $w^*\sim\mathcal N(\mathbf{0}_d,I_d)$, $x_i,x_{\mathrm q}\stackrel{\mathrm{i.i.d.}}{\sim}\mathcal N(\mathbf{0}_d,\Lambda)$, and $y_i=\langle w^*,x_i\rangle$, the LSA prediction is
\begin{equation}
\widehat Q(x_{\mathrm q}\mid P)
=
x_{\mathrm q}^\top
\Gamma^{-1}
\left(
\frac1M\sum_{i=1}^M y_i x_i
\right).
\label{eq:lsa_q_prediction}
\end{equation}
Moreover, in the noiseless realizable case,
\begin{equation}
\varepsilon_Q(M,N)
\triangleq
\mathbb E_{w^*,x_i,x_{\mathrm q}}\!\left[\bigl(\widehat Q(x_{\mathrm q}\mid P)-\langle w^*,x_{\mathrm q}\rangle\bigr)^2\right]
\le
\frac{(d+1)\operatorname{tr}(\Lambda)}{M}
+
\frac{(1+2d+d^2\kappa)\operatorname{tr}(\Lambda)}{N^2},
\label{eq:q_mse_expected}
\end{equation}
where $\kappa$ is the condition number of $\Lambda$.
\end{lemma}

Equation \eqref{eq:lsa_q_prediction} shows that the trained LSA layer implicitly computes a covariance-corrected in-context estimator. The factor $\Gamma^{-1}$ acts as a shrinkage version of $\Lambda^{-1}$ whose bias depends on the length $N=KT$. As $N\to\infty$, $\Gamma^{-1}\to\Lambda^{-1}$, and the estimator approaches the ordinary least-squares form. The two terms in \eqref{eq:q_mse_expected} have distinct origins: the $1/M$ term is the in-context estimation error on an unseen task, while the $1/N^2$ term is the training-length bias.

\subsection{End-to-end suboptimality of the induced policy}
\label{subsec:policy_bound}

Lemma~\ref{lem:q_prediction} bounds the error of one Q-value query under the matched Gaussian query model. Our experiments evaluate the policy obtained by acting with respect to the predicted Q-values along a full trajectory. To lift the prediction guarantee to a policy guarantee, we need an assumption that the Q-prediction error remains controlled on the query points encountered by the induced policy.

\paragraph{On-policy query calibration.}
Let $\widehat\pi$ be the policy induced by $\widehat Q$: $\widehat\pi_t(s)\in\arg\max_{a\in\mathcal A}\widehat Q_t(s,a)$.
We assume a constant $C_{\mathrm{on}}\ge 1$ such that, for every time $t\le T$,
\begin{equation}
\max_{\pi_t'\in\{\pi_t^*,\widehat\pi_t\}}
\mathbb E
\left[
\left(
\widehat Q_t(s_t,\pi_t'(s_t))
-
Q_t^*(s_t,\pi_t'(s_t))
\right)^2
\right]
\le
C_{\mathrm{on}}\varepsilon_Q(M,N),
\label{eq:on_policy_calibration}
\end{equation}
where the expectation is over the support trajectory, the unseen task, and the trajectory distribution $s_t\sim d_t^{\widehat\pi}$. This assumption is a direct on-policy transfer condition: it rules out the covariate-shift failure modes identified by \cite{zhang2023trained}. In the ideal matched-query setting, one may take $C_{\mathrm{on}}=1$; otherwise $C_{\mathrm{on}}$ measures the mismatch between the Gaussian query model in Lemma~\ref{lem:q_prediction} and the state-action queries encountered during rollout.

\begin{proposition}[End-to-end optimality gap]
\label{prop:policy_gap}
Let $\tau_{\mathrm{new}}$ be an unseen linear MDP drawn from $\mathcal T$, with horizon $T$, discount factor $\gamma\in(0,1]$, and optimal Q-functions linear in the features as in \eqref{eq:linear_q}. Let $\widehat Q$ be the LSA prediction obtained from a length-$M$ support trajectory, and let $\widehat\pi$ be the policy induced by $\widehat Q$. Under the conditions of Lemma~\ref{lem:q_prediction} and the on-policy calibration condition \eqref{eq:on_policy_calibration},
\begin{equation}
\mathbb E_{\tau_{\mathrm{new}}}
\left[
J_{\tau_{\mathrm{new}}}(\pi^*)
-
J_{\tau_{\mathrm{new}}}(\widehat\pi)
\right]
\le
C_{T,\gamma}\sqrt{C_{\mathrm{on}}\varepsilon_Q(M,N)},
\label{eq:policy_gap_general}
\end{equation}
where $C_{T,\gamma} \triangleq 2(1-\gamma^T)/(1-\gamma)$ and $C_{T,1} \triangleq 2T$.
Substituting \eqref{eq:q_mse_expected} gives
\begin{equation}
\mathbb E
\left[
J_{\tau_{\mathrm{new}}}(\pi^*)
-
J_{\tau_{\mathrm{new}}}(\widehat\pi)
\right]
\le
C_{T,\gamma}\sqrt{C_{\mathrm{on}}}
\left[
\frac{(d+1)\operatorname{tr}(\Lambda)}{M}
+
\frac{(1+2d+d^2\kappa)\operatorname{tr}(\Lambda)}{N^2}
\right]^{1/2}.
\label{eq:policy_gap_substituted}
\end{equation}
\end{proposition}

\paragraph{Implications.}
Substitute $M=K_{\mathrm{test}}T$ and $N=KT$. In the nearly undiscounted regime $\gamma\to 1$, \eqref{eq:policy_gap_substituted} implies
\begin{equation}
\left(
\mathbb E[J(\pi^*)-J(\widehat\pi)]
\right)^2
\lesssim
C_{\mathrm{on}}
\left[
\frac{4T(d+1)\operatorname{tr}(\Lambda)}{K_{\mathrm{test}}}
+
\frac{4(1+2d+d^2\kappa)\operatorname{tr}(\Lambda)}{K^2}
\right].
\label{eq:gap_squared_decomposition}
\end{equation}
Therefore, the number of support trajectories $K_{\mathrm{test}}$, the number of training trajectories $K$, and the horizon $T$ play different roles. Increasing $K_{\mathrm{test}}$ reduces the in-context estimation error, while increasing $K$ reduces the training-length bias. Increasing the horizon lengthens each trajectory, but it also increases the number of decision points over which Q-prediction errors can compound. These effects cancel in the second term but not in the first term in \eqref{eq:gap_squared_decomposition}. We provide the proof of Proposition~\ref{prop:policy_gap} and additional theoretical results in Appendix~\ref{sec:appendix_theory_details}.

%%%%%%%%%%%%%%%%%%%%%%%%%%%%%%%%%%%%%%%%%%%%%%%%%%%%%%
\section{Numerical experiments for performance evaluation}\label{Sec:Performance on Synthetic Data}
%%%%%%%%%%%%%%%%%%%%%%%%%%%%%%%%%%%%%%%%%%%%%%%%%%%%%%

We conduct numerical experiments to assess the performance of LLMs in sequential decision-making under MDP, POMDP, and APOMDP settings. We provide additional experiments in Appendix~\ref{sec:appendix_additional_results}, including comparisons with benchmarks such as DPT~\cite{DPT2023lee}, robustness tests, alternative few-shot demonstration analysis, and evaluation on the Darkroom task~\cite{DPT2023lee}.

\subsection{Task generation}\label{Secsec:Synthetic Data Generation}

\paragraph{MDP.} We consider an ``energy-management'' task that models the trade-off between ``charging'' and ``working'' using an MDP.
The system consists of $\mathrm{E} + 1$ discrete energy states $s \in \{0, 1, \ldots, \mathrm{E}\}$, where higher states correspond to greater stored energy.
At each decision step, the agent selects one action from $\mathcal{A}=\{0,1,2\}$. Actions $a=0$ and $a=2$ are two redundant charge actions: both replenish energy, transition from state $s$ to $\min(\mathrm{E}, s + 1)$ with probability $p$, and incur the same small cost $c$. The work action is $a=1$.
The work action consumes energy, transitioning from state $s$ to $\max(0, s-1)$ with probability $p$, and yields a reward proportional to the current energy level, $r = s / \mathrm{E}$.
The duplicated charge actions enlarge the action space, making action selection harder because two actions correspond to the same physical control.
We sample the probability $p \in [0.5, 1)$ uniformly to synthesize training and test tasks.\footnote{In all experiments, we set $\mathrm{E}=9$ and $c=-0.02$.}
While the theoretical analysis mentioned above studies a linear MDP, our experiments evaluate the proposed framework on tabular MDP settings.

\paragraph{POMDP.} The POMDP uses the same state dynamics and reward structure, but the agent receives only noisy observations. Each state emits one of $\mathrm{E}+1$ discrete observation symbols indexed by $\{0,1,\ldots,\mathrm{E}\}$, where the correct energy state is observed with probability $q$ and the remaining mass is uniformly distributed across other observations. We use $q \in \{0.5, 0.8, 1.0\}$ in our simulation study.

\paragraph{APOMDP.} We extend the setting by allowing both partial observability and model ambiguity. The agent faces the same charge-work trade-off but must reason over a set of plausible models. We use a simplified entropy-ball approximation based on Kullback--Leibler (KL) divergence: we first sample a transition-emission model, and additional latent models are Dirichlet perturbations around the first model, accepted only when $\mathrm{KL}(\mathrm{base}\,\|\,\mathrm{candidate}) \le 0.2$. We vary $|\mathcal{M}_\tau| \in \{1, 3, 5\}$.

For each training task $\tau$, we roll out 15 trajectories of length $T$ under its (approximately) optimal policy. We obtain these policies via backward induction for MDPs, belief-state backward induction for POMDPs, and robust backward induction using the $\alpha$-maximin expected utility ($\alpha$-MEU) criterion for APOMDPs~\cite{saghafian2018ambiguous}.

\subsection{Performance metric}\label{Secsec:Performance Metric}

We evaluate performance using the optimality gap:
\begin{equation}
\textrm{Optimality Gap} = \frac{1}{N_{\mathrm{eval}}} \sum_{\tau = 1}^{N_{\mathrm{eval}}} \frac{\mathrm{OPT}_{\tau}-\mathrm{OPT}^{\mathrm{Eval}}_{\tau}}{\mathrm{OPT}_{\tau}},
\end{equation}
where $N_{\mathrm{eval}}$ denotes the number of test tasks, $\mathrm{OPT}_{\tau}$ the reward under the optimal policy, and $\mathrm{OPT}^{\mathrm{Eval}}_{\tau}$ the reward under the evaluated policy. A smaller gap indicates closer-to-optimal performance.

\subsection{Estimation performance}\label{Secsec:Estimation Performance}

We compute the optimality gap on 100 test tasks by first computing rewards across evaluation trajectories within each task (30 for MDPs/POMDPs, 90 for APOMDPs), then averaging across tasks. Confidence intervals are two-sided Student's $t$ intervals around the mean over test tasks, scaled by its standard error.
Evaluation uses two few-shot support trajectories generated by the optimal policy. The ICL-only baseline uses the same pretrained base LLM without supervised fine-tuning.

\begin{figure}[thbp]
\centering
\begin{minipage}{0.3\textwidth}
\centering
\includegraphics[width=\linewidth]{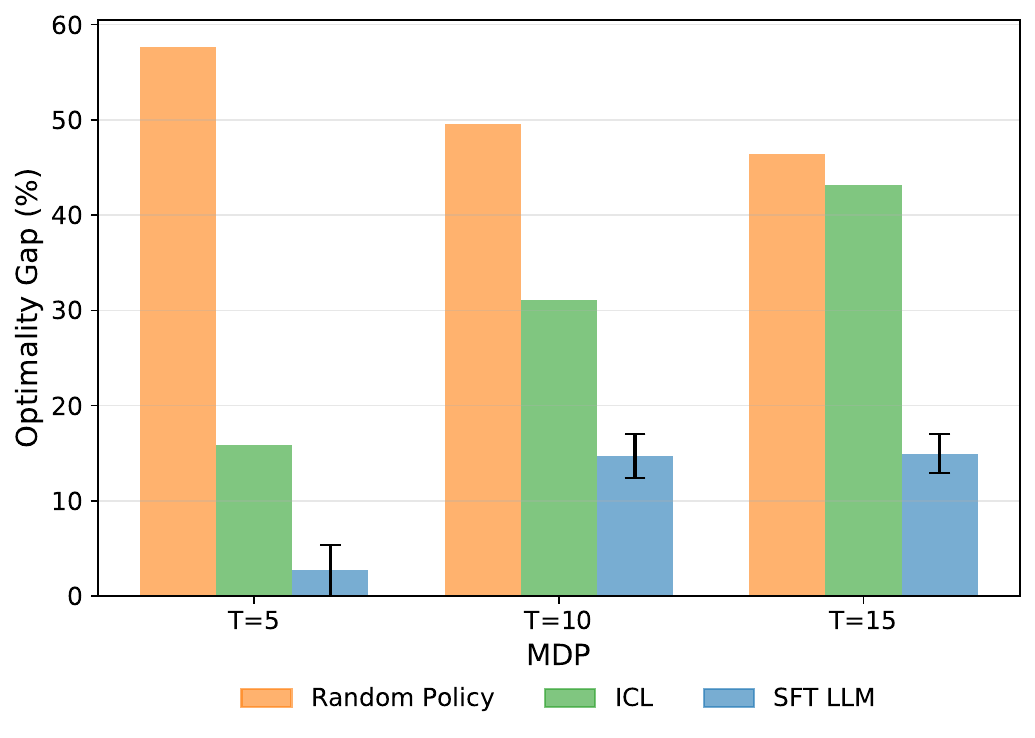}
\end{minipage}
\hfill
\begin{minipage}{0.32\textwidth}
\centering
\includegraphics[width=\linewidth]{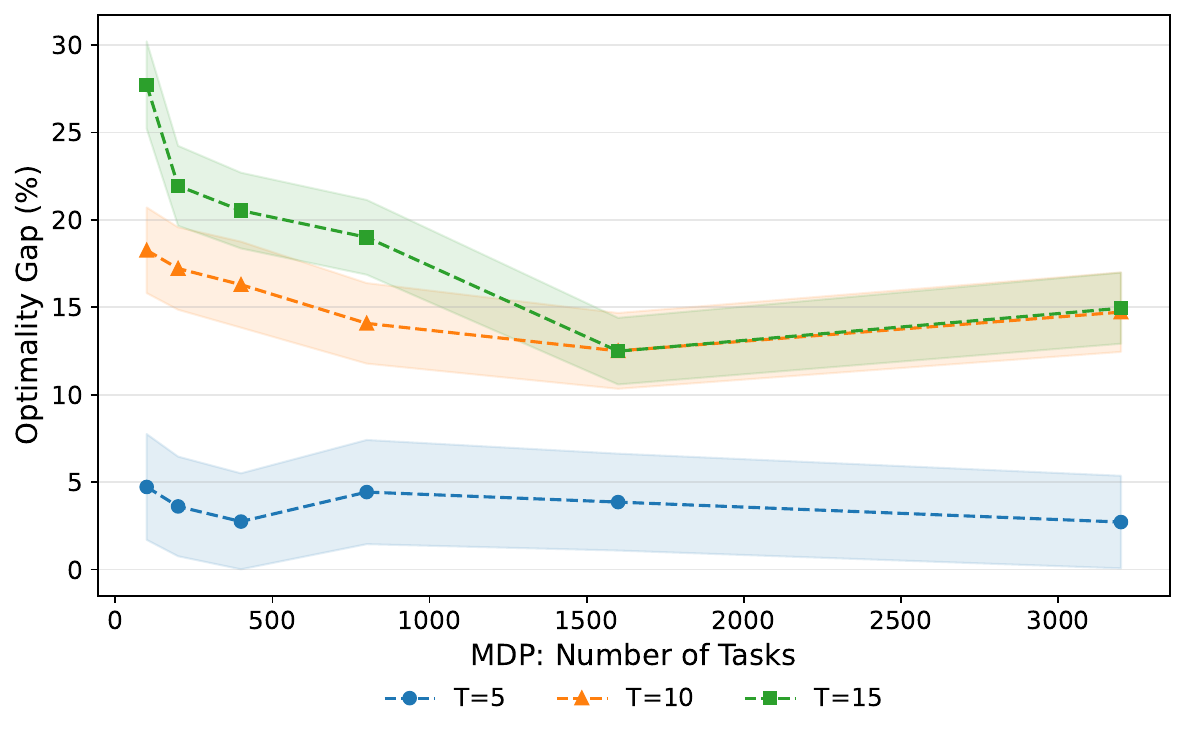}
\end{minipage}
\hfill
\begin{minipage}{0.32\textwidth}
\centering
\includegraphics[width=\linewidth]{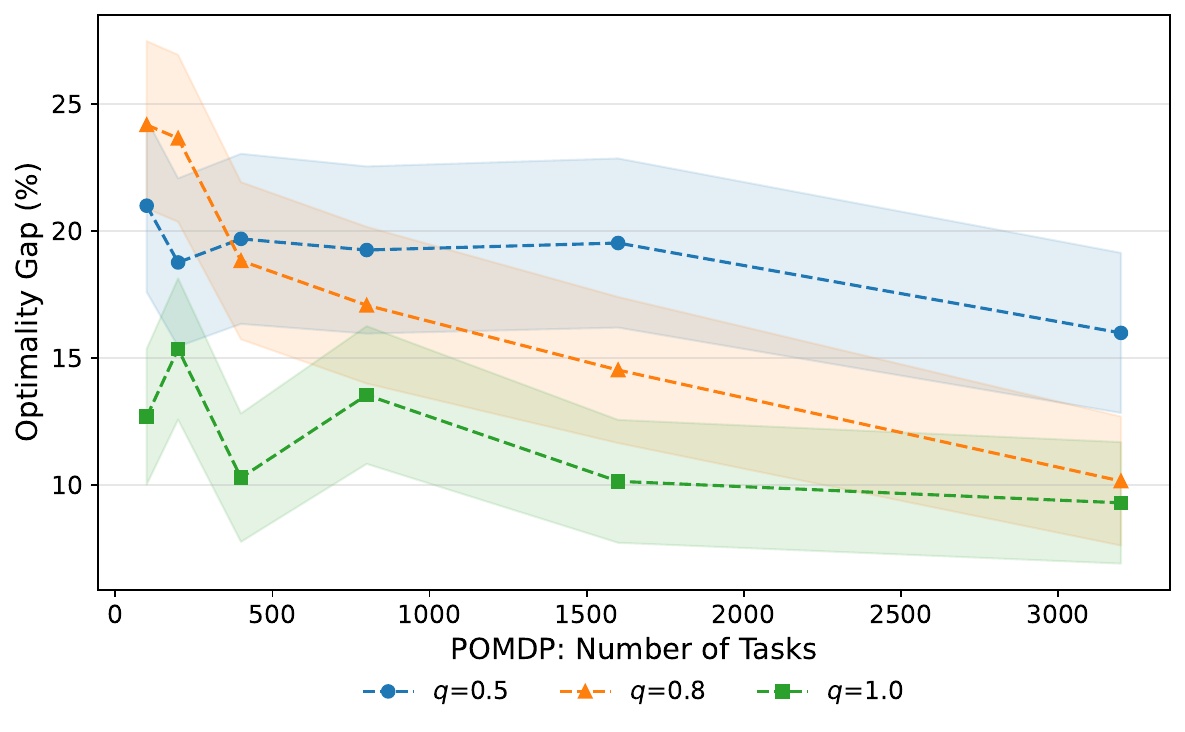}
\end{minipage}
\caption{Left: the optimality gaps of the random policy, the ICL policy, and the SFT policy (3{,}200 training tasks) across planning horizons $T \in \{5, 10, 15\}$ for MDPs. Center: the optimality gaps of the SFT policy versus the number of training tasks for MDPs across planning horizons $T \in \{5, 10, 15\}$. Right: the optimality gaps of the SFT policy versus the number of training tasks for POMDPs with planning horizon $T=10$ and the values of $q$ equal to 0.5, 0.8, and 1.0. The shaded areas represent the 95\% confidence interval.}
\label{figure:mdp-pomdp-gap}
\end{figure}

\paragraph{MDP.} Figure~\ref{figure:mdp-pomdp-gap} (left) shows the optimality gaps of the random policy, the ICL policy, and the SFT policy (3{,}200 training tasks) across planning horizons $T \in \{5, 10, 15\}$. The random policy achieves optimality gaps of 57.6\%, 49.6\%, and 46.4\% for $T = 5$, $10$, and $15$, respectively. The ICL policy reduces these gaps to 15.9\%, 31.1\%, and 43.2\%, which provides meaningful gains over random guessing. The SFT policy further reduces the gap to 3\%, 15\%, and 15\% for $T = 5$, $10$, and $15$, respectively. In particular, we see the largest improvements at longer horizons where in-context learning alone struggles.

Figure~\ref{figure:mdp-pomdp-gap} (center) shows the optimality gap as a function of the number of training tasks (100, 200, 400, 800, 1{,}600, and 3{,}200). For $T = 10$ and $15$, the gap decreases as the number of training tasks increases. For example, at $T = 15$, the gap decreases from about 30\% to 15\%. For $T=5$, the SFT policy already performs well with few training tasks. The larger improvements for longer horizons indicate that fine-tuning the LLM is valuable for complex, longer-horizon problems.

\paragraph{POMDP.}
Figure~\ref{figure:mdp-pomdp-gap} (right) shows the optimality gaps under different observation probabilities with planning horizon $T=10$. When $q=1$ (full observability), the gap is smallest. As $q$ decreases, belief states become less informative and the gaps increase. Additional fine-tuning data are especially valuable in partially observed environments, with the fully observed case showing modest gains from additional training tasks.

\begin{figure}[thbp]
\centering
\begin{minipage}{0.32\textwidth}
\centering
\includegraphics[width=\linewidth]{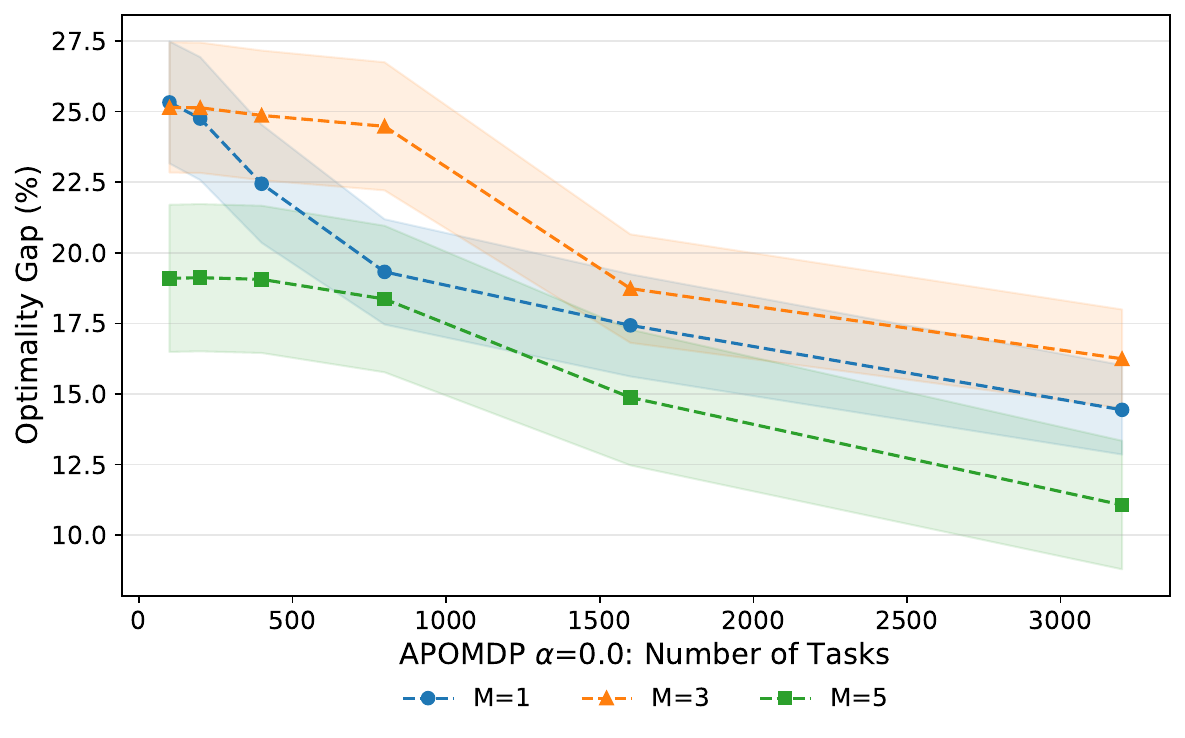}
\end{minipage}
\hfill
\begin{minipage}{0.32\textwidth}
\centering
\includegraphics[width=\linewidth]{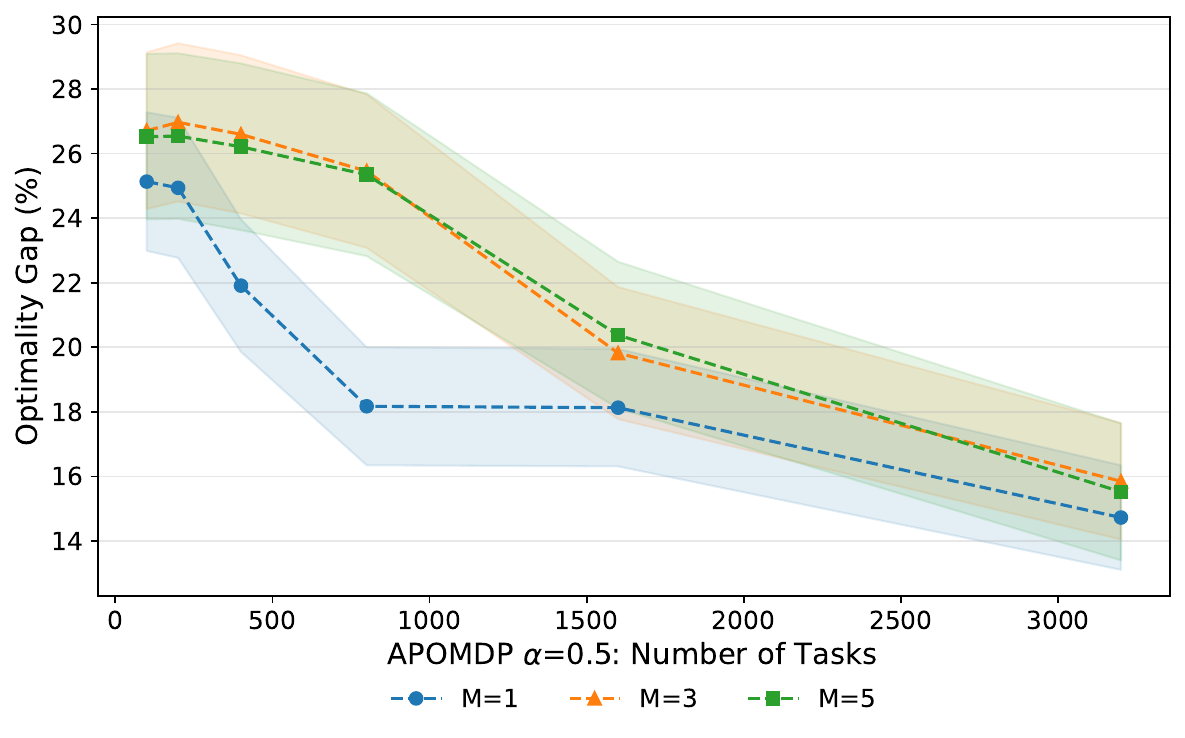}
\end{minipage}
\hfill
\begin{minipage}{0.32\textwidth}
\centering
\includegraphics[width=\linewidth]{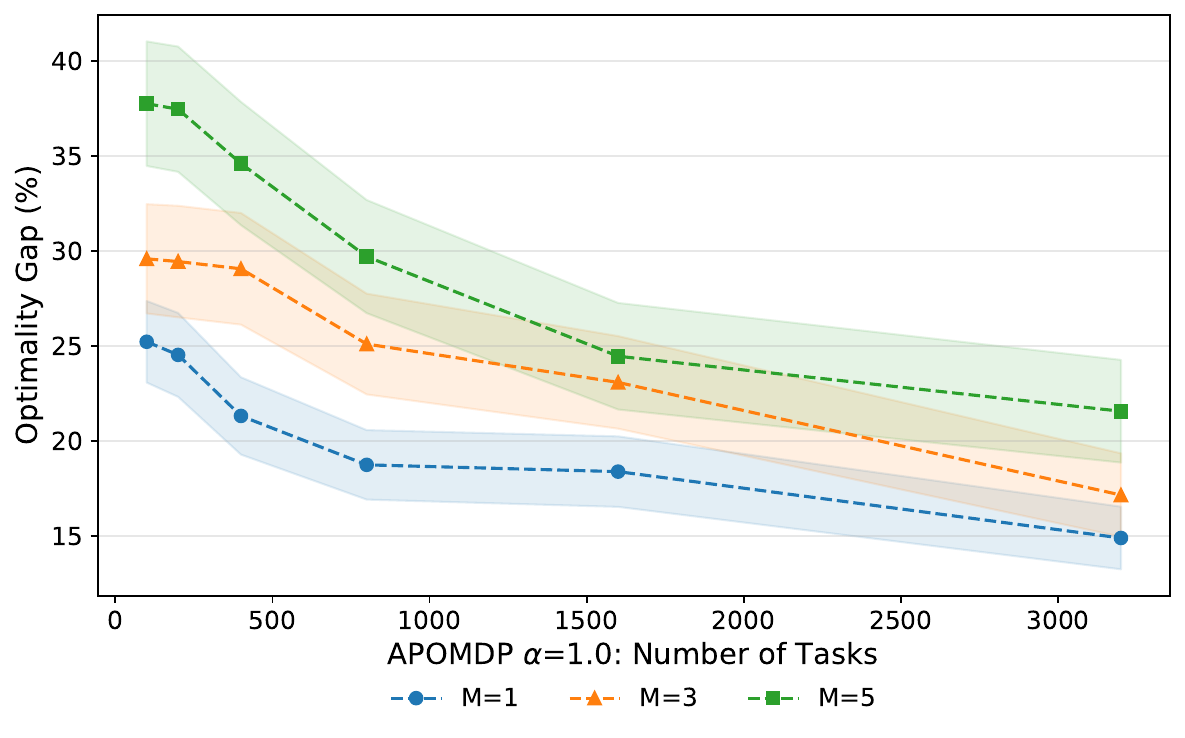}
\end{minipage}
\caption{Optimality gaps versus the number of training tasks for APOMDPs. Planning horizon $T=10$ and $q=0.8$. The three lines correspond to ambiguity set sizes $|\mathcal{M}_\tau| \in \{1, 3, 5\}$. The three panels correspond to $\alpha = 0.0$ (left), $\alpha = 0.5$ (center), and $\alpha = 1.0$ (right). The shaded areas represent the 95\% confidence interval.}
\label{figure:apomdp-gap}
\end{figure}

\paragraph{APOMDP.} Figure~\ref{figure:apomdp-gap} shows the optimality gaps for APOMDPs. For a given number of models, the gap decreases as the number of training tasks increases. The effect of ambiguity set size $|\mathcal{M}_\tau|$, however, depends on the ambiguity attitude $\alpha$. Under the optimistic criterion ($\alpha = 0$), a larger $|\mathcal{M}_\tau|$ expands the maximization space, providing the agent with a richer set of candidate models to exploit. This broader exploration scope facilitates better generalization, and the optimality gap decreases when $|\mathcal{M}_\tau|$ increases. In contrast, under the pessimistic criterion ($\alpha = 1$), a larger $|\mathcal{M}_\tau|$ tightens the worst-case constraint by minimizing over more models and increases the optimality gap. The neutral criterion ($\alpha = 0.5$) balances these effects, and the gap is relatively stable when $|\mathcal{M}_\tau|$ increases. These results suggest that fine-tuning LLMs allows them to handle model ambiguity, and that the benefits depend on the ambiguity attitude and the size of the ambiguity set.

\paragraph{Comparison against DPT.}
We compare our fine-tuned LLM against DPT~\cite{DPT2023lee} with 3{,}200 training tasks. Across all three settings, our approach achieves smaller optimality gaps. In the MDP setting, the gap decreases from 43\% (DPT) to 15\% (SFT) at $T=15$. In the POMDP setting, the gap decreases from 25\% (DPT) to 10\% (SFT) at $q=0.8$. In the APOMDP setting, our approach maintains a roughly 10-percentage-point advantage across all ambiguity attitudes. In particular, the gaps for DPT grow faster than those of our approach under the pessimistic criterion ($\alpha=1.0$) when the ambiguity set size increases. We provide more details in Appendix~\ref{sec:appendix_dpt}.

\paragraph{Robustness to out-of-distribution (OOD) test conditions.}
We evaluate the robustness of our approach by testing on parameters outside the training distribution. For MDPs, the model trained with $p \in [0.5, 1)$ and tested on $p \in [0.3, 0.5)$ achieves low out-of-distribution gaps at all three planning horizons ($T=5$, $10$, and $15$). For POMDPs, the model trained with $q=0.8$ achieves nearly identical gaps when tested at $q \in \{0.6, 0.7, 0.9\}$. For APOMDPs, the model trained at $\alpha=0.5$ transfers well to $\alpha \in \{0.0, 1.0\}$, with gaps increasing slightly. We provide more details in Appendix~\ref{sec:appendix_robustness}.

\paragraph{Quality of few-shot support trajectories.}
Our results suggest that replacing rollouts under the optimal policy with rollouts under the random policy in the few-shot support trajectories affects the performance. The difference is most pronounced at $T=5$: with optimal support trajectories, the SFT LLM achieves an optimality gap of about 3\%, while with random support trajectories the gap remains around 20\% even at 3{,}200 training tasks. We provide more details in Appendix~\ref{sec:appendix_fewshot}.

\paragraph{Performance on Darkroom tasks.}
To further evaluate our approach, we use the Darkroom task~\cite{DPT2023lee}, a $10\times10$ deterministic grid-world with a hidden goal cell. The agent has five actions (the four cardinal moves and a stay action), receives a reward of $+1$ only when it executes the stay on the goal cell, and 0 otherwise. The task is challenging because rewards are sparse and the goal location is unobservable, so the LLM must infer it from the few-shot demonstrations alone. We split the 100 possible goals into 80 fine-tuning goals and 20 held-out test goals. Our approach achieves approximately 95\% of oracle performance on 20 test goals, far exceeding the random policy. We provide more details in Appendix~\ref{sec:appendix_darkroom}.

%%%%%%%%%%%%%%%%%%%%%%%%%%%%%%%%%%%%%%%%%%%%%%%%%%%%%%
\section{Discussion}\label{sec:discussion}
%%%%%%%%%%%%%%%%%%%%%%%%%%%%%%%%%%%%%%%%%%%%%%%%%%%%%%

In this paper, we study the ability of LLMs to support sequential decision-making. We propose a supervised fine-tuning framework that adapts a pretrained LLM using parameter-efficient QLoRA and offline, oracle-labeled sequential decision-making data. This process yields a history-conditioned policy that can generalize to new tasks via few-shot support trajectories. Empirically, we move beyond typical fully observed MDPs and evaluate LLMs on two more challenging decision-making environments that allow partial observability and model ambiguity: POMDPs and APOMDPs. Our experiments show that fine-tuned LLMs consistently perform well across all three settings, with the largest gains in long-horizon, partially observed, and model-ambiguous settings. The fine-tuned LLM also compares favorably against DPT across all benchmarks and demonstrates robust generalization to out-of-distribution test conditions. Our theoretical results interpret the fine-tuned attention mechanism as implicitly estimating optimal Q-function parameters, and use this interpretation to create an end-to-end suboptimality bound for the induced policy that separates the in-context estimation error from the training-length bias.

We discuss the limitations of our work in Appendix~\ref{sec:appendix_limitations}. Several promising directions remain for future work. A primary motivation of our work is applying this framework to real-world clinical decision-making using electronic health record data, where offline observational data are abundant. Using RL or other approaches to improve LLM capabilities in understanding causation under ambiguity and proposing effective dynamic treatment recommendations (see, e.g.,~\cite{saghafian2023ambiguous}) can be an important path toward achieving this goal. Furthermore, reducing the reliance on oracle supervision would significantly broaden applicability to domains where optimal policies are hard to obtain. Extending the theoretical and empirical analysis to multimodal settings is another natural next step, as it can enable further progress toward LLM preparation for real-world decision-making applications such as those in healthcare. Finally, exploring whether the natural language capabilities of pretrained LLMs could incorporate human intuition alongside numerical trajectories (e.g., to create powerful centaurs~\cite{saghafian2025insight, saghafian2024effective,dean2022algorithm}) may open new avenues for knowledge transfer in complex, high-dimensional environments.

%%%%%%%%%%%%%%%%%%%%%%%%%%%%%%%%%%%%%%%%%%%%%%%%%%%%%%
% Acknowledgments (auto-hidden in anonymized submission)
%%%%%%%%%%%%%%%%%%%%%%%%%%%%%%%%%%%%%%%%%%%%%%%%%%%%%%

\begin{ack}
This work was partially supported by a grant from Amazon AWS and Harvard Data Science Initiative (PI: Saghafian).
\end{ack}

\bibliographystyle{unsrt}
{\small \bibliography{bib}}

\clearpage
\newpage

\appendix

\section{Additional related work}\label{sec:appendix_additional_related_work}

\paragraph{Causal inference and LLMs.} Causal foundation models have been developed to estimate treatment effects (see, e.g.,~\cite{zhang2023towards}), but these studies typically focus on single-time-step settings. Our study addresses the gap in using LLMs for sequential causal decision-making, particularly in ADTR settings with unobservable confounders. A survey of related studies in this stream can be found in~\cite{liu2024large}.

\paragraph{Offline reinforcement learning.} Offline RL approaches learn policies from previously collected data without online interaction~\cite{levine2020offline}. Some approaches include Conservative Q-Learning~\cite{kumar2020conservative}, which regularizes Q-values to avoid overestimation of out-of-distribution actions, pessimistic offline RL with provable guarantees~\cite{jin2021pessimism}, and minimalist approaches that add behavioral cloning regularization~\cite{fujimoto2021minimalist}. Our study shares the offline data setting with this literature but proposes a different approach. Rather than learning a policy through explicit value-function optimization, we fine-tune an LLM to generate optimal actions in-context, which enables policy transfer to new tasks via few-shot support trajectories.

\paragraph{Meta-reinforcement learning.} Meta-reinforcement learning trains agents that can adapt to new tasks from a shared task distribution. \cite{duan2016rl} and~\cite{wang2016learning} train recurrent networks on sequences of episodes so that the hidden state implicitly encodes a learning algorithm. \cite{finn2017model} instead learns initializations amenable to fast gradient-based adaptation. Our study leverages pretrained language model representations and uses supervised fine-tuning rather than RL-based meta-training.

\section{Definitions: MDP and POMDP}\label{sec:appendix_standard_definitions}

We provide the formal definitions for the MDP and POMDP discussed in Section~\ref{sec:special_cases}.

\paragraph{MDP [full observability].}
We set $X_t = S_t \in \mathcal{S}_\tau$ and model the task as
\[
\mathcal{E}^{\mathrm{MDP}}_\tau
=\big\langle \mathcal{S}_\tau,\; \mathcal{A}_\tau,\; P_\tau,\; R_\tau,\; \rho_\tau,\; T_\tau,\; \gamma \big\rangle,
\]
where $P_\tau(\cdot \mid s,a)$ is the state-transition kernel, $R_\tau(s,a)$ the one-step reward, and $\rho_\tau$ the initial-state distribution. The Markov property implies
$\mathbb{P}(S_{t+1} \mid H_t,a_t)=P_\tau(S_{t+1} \mid S_t,a_t)$ and $r_t=R_\tau(S_t,a_t)$.
Without loss of optimality, one may restrict to Markov policies $\pi_t(a\mid H_t)=\pi_t(a\mid S_t)$, and to stationary $\pi(a\mid S)$ when the model is time-homogeneous.

\paragraph{POMDP [partial observability].}
When we only observe $O_t\in\mathcal{O}_\tau$ generated from a latent state $S_t\in\mathcal{S}_\tau$, we set $X_t=O_t$ and define
\[
\mathcal{E}^{\mathrm{POMDP}}_\tau
=\big\langle \mathcal{S}_\tau,\; \mathcal{A}_\tau,\; \mathcal{O}_\tau,\; P_\tau,\; Q_\tau,\; R_\tau,\; \rho_\tau,\; T_\tau,\; \gamma \big\rangle,
\]
where $P_\tau$ is the state-transition kernel and $Q_\tau(o\mid s,a)$ is the observation kernel.
It is often convenient to work with the belief vector $b_t \in \Delta(\mathcal{S}_\tau)$ over latent states, which summarizes all information in the history $H_t$. The belief is updated via a Bayes operator $B_{\tau}$, i.e., $b_{t+1} = B_{\tau}(b_t, a_t, o_{t+1})$, which combines the prior belief $b_t$, the state transition dynamics, and the observation likelihood to produce a new posterior using Bayes' rule. The belief vector $b_t$ serves as a sufficient statistic that restores a Markov property in the belief space. We restrict attention to belief-based policies $\pi_t(a \mid H_t) = \pi_t(a \mid b_t)$. MDPs arise as special cases with full observability by setting $\mathcal{O}_\tau = \mathcal{S}_\tau$ and $Q_\tau(o \mid s,a) = \mathbf{1}\{o = s\}$.

\section{Supervised fine-tuning details}\label{sec:appendix_qlora}

We fine-tune an open-source LLM using a parameter-efficient QLoRA setup~\cite{dettmers2024qlora}. Specifically, for each linear module $W\!\in\!\mathbb{R}^{d_{\mathrm{out}}\times d_{\mathrm{in}}}$ in attention and multi-layer perceptron blocks, we learn a low-rank update $\Delta W\!=\!B A$ with $A\!\in\!\mathbb{R}^{r\times d_{\mathrm{in}}}$ and $B\!\in\!\mathbb{R}^{d_{\mathrm{out}}\times r}$ and keep the original parameters $W_0$ of the LLM frozen. Here $d_{\mathrm{in}}$ and $d_{\mathrm{out}}$ are the input and output feature dimensions of the linear module and $r$ is the rank of the low-rank update. The updated parameters are given by $W = W_0 + \Delta W = W_0 + BA$, where $r \ll \min\{d_{\mathrm{in}}, d_{\mathrm{out}}\}$. Only the adapter parameters $(A,B)$ are updated to minimize the loss $\mathcal{L}_{\mathrm{SFT}}(\theta)$ defined in Eq.~\eqref{eq:sft_loss}. Because the base model weights remain fixed, we substantially reduce the number of trainable parameters and computational resources required.

Throughout our experiments, we use Meta's Llama-2-7B as the base model, which strikes a favorable balance between capacity and trainability. We use the Paged AdamW 32-bit optimizer with a learning rate of $2 \times 10^{-4}$, a cosine learning rate schedule, and a warmup ratio of 0.05. For the low-rank adaptation (LoRA) component, the adapter rank is set to $r = 64$ with LoRA $\alpha = 64$ and dropout of 0.1, applied to all attention and multilayer perceptron projection layers. We train for three epochs, evaluating on the validation set (10\% training tasks) at each epoch. All training is conducted on a single NVIDIA A100 GPU with 80\,GB of memory using QLoRA with 4-bit NormalFloat (NF4) quantization and bfloat16 mixed-precision computation to reduce memory usage and accelerate convergence. The wall-clock cost of a single fine-tuning run together with its evaluation ranges from a few hours to roughly three to four days, depending on the number of training tasks. APOMDP runs are more expensive than MDP and POMDP runs because of computing the optimal policy by backward induction. Including preliminary and failed runs that did not appear in the final paper, the full project consumed approximately two to three months to run.

\paragraph{Licenses of existing assets.} The base model Llama-2-7B is used under the Llama 2 Community License Agreement~\cite{touvron2023llama}, which permits research use with attribution. We use the QLoRA implementation~\cite{dettmers2024qlora} and the \texttt{bitsandbytes} library, both released under the MIT license, together with HuggingFace \texttt{transformers} and \texttt{peft} (Apache~2.0) and PyTorch (BSD-3-Clause). The Darkroom task in Appendix~\ref{sec:appendix_darkroom} is reimplemented from the description in \cite{DPT2023lee}.

\section{Detailed pseudocode}\label{sec:appendix_pseudocode}

We show detailed pseudocode for the key components of our framework: construction of the fine-tuning dataset, the supervised fine-tuning process with QLoRA, and the evaluation procedure.

\begin{algorithm}[thbp]
\caption{Training and evaluation pipeline}
\label{alg:pipeline_simplified}
\begin{algorithmic}[1]
\Require Training tasks.
\Ensure Fine-tuned model parameters $\theta$ and performance evaluation.
\State \textbf{(Data)} For each $\tau_{\mathrm{train}} \sim \mathcal{T}$, roll out multiple demonstration trajectories $\mathcal{C}_{\tau}$ under the optimal policy, serialize them into strings $s$, and collect $\{\mathcal{C}_{\tau}\}$.
\State \textbf{(ICL)} For each test task $\tau_{\mathrm{test}}\sim \mathcal{T}$, construct textual context $\mathrm{ctx}_{\tau}$ from few-shot support trajectories and generate actions.
\State \textbf{(SFT)} Fine-tune the LLM with QLoRA by minimizing $\mathcal{L}_{\mathrm{SFT}}(\theta)$ using $s \in \{\mathcal{C}_{\tau}\}$ (Eq.~\eqref{eq:sft_loss}).
\State \textbf{(Eval)} For test tasks $\{\tau_{\mathrm{test}}\}$, compute reward and performance gap.
\end{algorithmic}
\end{algorithm}

\begin{algorithm}[thbp]
\caption{Construct training data for supervised fine-tuning}
\label{alg:construct_data}
\begin{algorithmic}[1]
\Require Task distribution $\mathcal{T}$; number of training tasks $N_{\mathrm{ft}}$; number of trajectories $K$ per task; horizon $T$.
\Ensure Fine-tuning dataset $\{\mathcal{C}_\tau\} = \bigcup_{\tau} \mathcal{C}_\tau$ of serialized trajectories.
\For{$i = 1,\dots,N_{\mathrm{ft}}$}
  \State Draw a task $\tau$ from $\mathcal{T}$ (MDP/POMDP/APOMDP)
  \State Compute (approximately) optimal policy for $\tau$
  \For{$k = 1,\dots,K$}
    \State Roll out a trajectory of length $T$: $(X_t, a_t, r_t)_{t=1}^{T}$
    \State Serialize the trajectory into a string via a deterministic encoder $\mathsf{E}$
    \State Add the serialized trajectory to $\mathcal{C}_\tau$
  \EndFor
\EndFor
\State \Return $\{\mathcal{C}_\tau\}$
\end{algorithmic}
\noindent\footnotesize\textbf{Note.} A serialized trajectory is a textual encoding of the sequence $(X_t, a_t, r_t)_{t=1}^{T}$ under a fixed schema defined by the encoder $\mathsf{E}$. For example, $\mathsf{E}$ may produce a string of the form \texttt{<O\_1> 3, <A\_1> 1, <R\_1> 0.33, <O\_2> 2, \ldots}, where \texttt{<O\_t>}, \texttt{<A\_t>}, and \texttt{<R\_t>} are characters indicating the observation, action, and reward at period $t$, respectively.
\end{algorithm}

\begin{algorithm}[thbp]
\caption{Supervised fine-tuning with QLoRA}
\label{alg:sft}
\begin{algorithmic}[1]
\Require Base model parameters $W_0$; adapter rank $r$; fine-tuning dataset $\{\mathcal{C}_\tau\}$.
\Ensure Fine-tuned model.

\State For $s \in \{\mathcal{C}_\tau\}$, tokenize $s \mapsto (u_1,\dots,u_L)$

\State For linear module $W \in \mathbb{R}^{d_{\mathrm{out}}\times d_{\mathrm{in}}}$, parameterize $W = W_0 + BA$ with $A \in \mathbb{R}^{r\times d_{\mathrm{in}}}$ and $B \in \mathbb{R}^{d_{\mathrm{out}}\times r}$, where $W_0$ is the frozen base model parameter and $BA$ is the trainable low-rank update

\State Take a gradient step on $\theta=(A, B)$ to minimize the loss function:
\Statex \hspace{1em} $\mathcal{L}_{\mathrm{SFT}}(\theta) = \mathbb{E}_{\,s \in \{\mathcal{C}_\tau\}}\left[-\sum_{i=1}^{L} \log p_{\theta}(u_i \mid u_{<i})\right]$ (Eq.~\eqref{eq:sft_loss})

\State \Return Fine-tuned model
\end{algorithmic}
\noindent\footnotesize\textbf{Note.} Here $d_{\mathrm{in}}$ and $d_{\mathrm{out}}$ are the input and output feature dimensions of the linear module.
\end{algorithm}

\begin{algorithm}[thbp]
\caption{Evaluation via rollouts}
\label{alg:evaluation}
\begin{algorithmic}[1]
\Require LLM model; task distribution $\mathcal{T}$; horizon $T$; number of trajectories $K_{\mathrm{eval}}$ per test task.
\Ensure Reward and performance gap.
\State Draw a task $\tau_{\mathrm{test}}$ from $\mathcal{T}$ and compute (approximately) optimal policy for $\tau_{\mathrm{test}}$
  \State Roll out few-shot support trajectories $\mathcal{D}_\tau$ 
  \State Serialize trajectories into a textual context $\mathrm{ctx}_\tau = \mathsf{E}(\mathcal{D}_\tau)$
  \For{$i=1,\dots,K_{\mathrm{eval}}$ and $t = 1,\dots,T$}
    \State Form the input $(H_t, \mathrm{ctx}_\tau)$ and decode the next action token(s) to obtain $a_t$
    \State Execute $a_t$ in $\tau_{\mathrm{test}}$; observe reward $r_t$ and next observation
    \State Update history $H_{t+1} \gets H_t \cup \{(a_t,r_t,X_{t+1})\}$
  \EndFor
\State \Return Reward and performance gap
\end{algorithmic}
\noindent\footnotesize\textbf{Note.} We restrict the outputs of the LLM to action tokens and select the one with the highest probability. 
\end{algorithm}

\section{Additional theoretical results}\label{sec:appendix_theory_details}

\subsection{Proof of Proposition~\ref{prop:policy_gap}}
\label{sec:appendix_proof_e2e_subopt}

\begin{proof}
For each time $t$ and state $s$, greediness of $\widehat\pi_t$ with respect to $\widehat Q_t$ implies
\begin{align*}
V_t^*(s)-Q_t^*(s,\widehat\pi_t(s))
&=
Q_t^*(s,\pi_t^*(s))-Q_t^*(s,\widehat\pi_t(s))
\\
&\le
\left|
\widehat Q_t(s,\pi_t^*(s))
-
Q_t^*(s,\pi_t^*(s))
\right|
+
\left|
\widehat Q_t(s,\widehat\pi_t(s))
-
Q_t^*(s,\widehat\pi_t(s))
\right|.
\end{align*}
The finite-horizon performance-difference identity gives
\[
J_{\tau_{\mathrm{new}}}(\pi^*)-J_{\tau_{\mathrm{new}}}(\widehat\pi)
=
\sum_{t=1}^T
\gamma^{t-1}
\mathbb E_{s_t\sim d_t^{\widehat\pi}}
\left[
V_t^*(s_t)-Q_t^*(s_t,\widehat\pi_t(s_t))
\right].
\]
Combining the previous two displays, applying Jensen's inequality, and using \eqref{eq:on_policy_calibration} for both $\pi_t^*$ and $\widehat\pi_t$ yields
\[
\mathbb E
\left[
V_t^*(s_t)-Q_t^*(s_t,\widehat\pi_t(s_t))
\right]
\le
2\sqrt{C_{\mathrm{on}}\varepsilon_Q(M,N)} .
\]
Summing the geometric series over $t=1,\ldots,T$ gives \eqref{eq:policy_gap_general}. Equation \eqref{eq:policy_gap_substituted} follows by substituting \eqref{eq:q_mse_expected}.
\end{proof}

\subsection{Numerical validation of Proposition~\ref{prop:policy_gap}}\label{sec:appendix_simulation_validation}

To complement the analytical bounds, we run a simulation that tests the simulation-lemma step of Proposition~\ref{prop:policy_gap}. The objective is to verify that, even when $\varepsilon_Q$ is measured empirically rather than upper-bounded analytically, the cumulative gap satisfies
\begin{equation}\label{eq:appendix_simlemma_target}
J_{\tau_{\mathrm{new}}}(\pi^*) - J_{\tau_{\mathrm{new}}}(\hat{\pi})
\;\leq\; C_{T,\gamma}\sqrt{C_{\mathrm{on}}}\,\sqrt{\widehat{\varepsilon}_Q},
\end{equation}
where $\widehat{\varepsilon}_Q$ is the empirical mean of $(\hat{Q}-Q^*)^2$ along
the same trajectory used to evaluate the gap.

\paragraph{Setup.} We consider a linear MDP with feature
dimension $d=10$, $|\mathcal{A}|=5$ actions per step, horizon $T=10$,
and discount $\gamma=0.95$. For each task $\tau$ we draw $w^*_\tau\sim
\mathcal{N}(\mathbf{0}_d,I_d)$ independently. Per-step candidate features are i.i.d.\
$\phi_{t,a}\sim\mathcal{N}(\mathbf{0}_d,\Lambda)$, and the in-context demonstrations are
noiseless pairs $(x_i,\langle w^*_\tau,x_i\rangle)$ with $x_i\sim\mathcal{N}(\mathbf{0}_d,\Lambda)$.
The closed-form predictor $\Gamma^{-1}\bigl(\frac{1}{M}\sum_i y_i x_i\bigr)$
from Eq.~\eqref{eq:lsa_q_prediction} is applied with
$\Gamma=(1+1/N)\Lambda+(\operatorname{tr}\Lambda/N)I_d$. The empirical gap is
$\widehat{\mathrm{gap}}=\sum_{t=1}^{T}\gamma^t[Q^*(s_t,\pi^*(s_t))-Q^*(s_t,\hat{\pi}(s_t))]$,
and $\widehat{\varepsilon}_Q$ is the mean of $(\hat{Q}-Q^*)^2$ over the
$T\cdot|\mathcal{A}|=50$ query points. All quantities are
averaged over $500$ independent tasks per configuration.
We jointly vary three key parameters that could affect the gap and $\varepsilon_Q$:
\begin{itemize}[noitemsep, leftmargin=*]
\item Support trajectory length $M\in\{10,20,50,100,200,500,1000\}$, which controls
$\widehat{\varepsilon}_Q$ along each curve.
\item Training length $N\in\{100,1000,10000\}$, which enters
through $\Gamma^{-1}$ and controls the training-length bias.
\item Feature covariance via a diagonal $\Lambda$ with log-spaced eigenvalues
between $1/\kappa$ and $1$, for $\kappa\in\{1,5,25\}$ (isotropic, mild, and
strong anisotropy).
\end{itemize}

\paragraph{Results.}
Figure~\ref{fig:appendix_simlemma} plots the empirical gap against
$\widehat{\varepsilon}_Q$ over all $7\times 3\times 3=63$ configurations,
with one panel per $\Lambda$ and one curve per $N$. 
We see that every empirical curve lies strictly below the dashed reference
line~\eqref{eq:appendix_simlemma_target}: the simulation-lemma step holds in
every configuration, including the strongly anisotropic regime $\kappa=25$ in
which $\operatorname{tr}(\Lambda^{-1})$ is largest.

\begin{figure}[thbp]
\centering
\includegraphics[width=\linewidth]{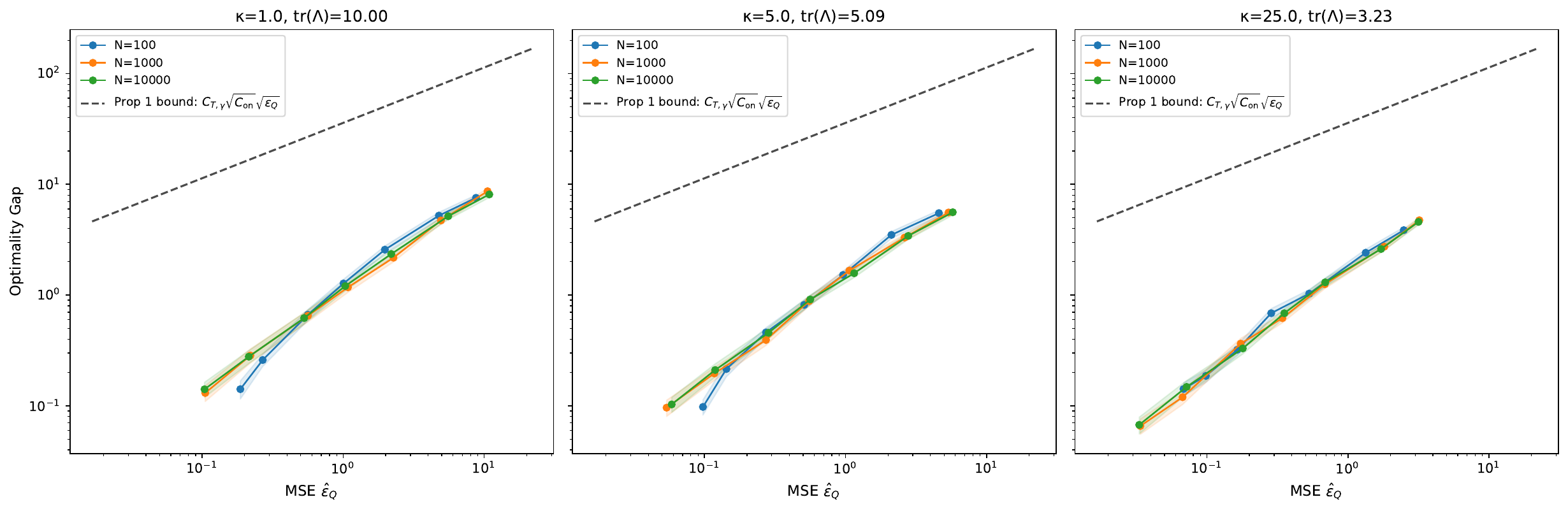}
\caption{
Empirical gap versus measured $\widehat{\varepsilon}_Q$ across
$\Lambda$, training length $N$, and support trajectory length
$M$. All curves remain strictly below the dashed reference line
$C_{T,\gamma}\sqrt{C_{\mathrm{on}}}\sqrt{\widehat{\varepsilon}_Q}$. The shaded areas represent the
$95\%$ confidence intervals.}
\label{fig:appendix_simlemma}
\end{figure}

\subsection{Joint sample complexity for an \texorpdfstring{$\varepsilon$}{epsilon}-optimal policy}
\label{subsec:sample_complexity}

The decomposition \eqref{eq:gap_squared_decomposition} gives a direct sufficient condition for achieving a target policy quality.

\begin{proposition}[Sample complexity for $\varepsilon$-optimal policies]
\label{prop:sample_complexity}
Fix $\varepsilon>0$. Under the conditions of Proposition~\ref{prop:policy_gap}, with $\gamma\to 1$, the induced policy $\widehat\pi$ satisfies
\[
\mathbb E
\left[
J_{\tau_{\mathrm{new}}}(\pi^*)
-
J_{\tau_{\mathrm{new}}}(\widehat\pi)
\right]
\le
\varepsilon
\]
whenever
\begin{equation}
K_{\mathrm{test}}
\ge
\frac{
8C_{\mathrm{on}}T(d+1)\operatorname{tr}(\Lambda)
}{\varepsilon^2}
\qquad
\text{and}
\qquad
K
\ge
\frac{
\sqrt{
8C_{\mathrm{on}}(1+2d+d^2\kappa)\operatorname{tr}(\Lambda)
}
}{\varepsilon}.
\label{eq:sample_complexity}
\end{equation}
In particular, for isotropic features $\Lambda=cI_d$ with $c=\Theta(1)$,
\begin{equation}
K_{\mathrm{test}}
=
O\!\left(\frac{C_{\mathrm{on}}T d^2}{\varepsilon^2}\right),
\qquad
K
=
O\!\left(\frac{\sqrt{C_{\mathrm{on}}}\,d^{3/2}}{\varepsilon}\right).
\end{equation}
\end{proposition}

\begin{proof}
Set each of the two terms in \eqref{eq:gap_squared_decomposition} to be at most $\varepsilon^2/2$ and solve for $K_{\mathrm{test}}$ and $K$. The isotropic specialization follows by substituting $\Lambda=cI_d$, which gives $\kappa=1$ and $\operatorname{tr}(\Lambda)=cd$.
\end{proof}

Proposition~\ref{prop:sample_complexity} highlights two qualitative differences between support trajectories and training trajectories. First, the required number of support trajectories grows with the rollout horizon $T$, while the required number of training rollouts per task does not after substituting $N=KT$. Second, $K_{\mathrm{test}}$ scales as $1/\varepsilon^2$, whereas $K$ scales as $1/\varepsilon$. This result suggests that, in the regime captured by the theory, increasing the amount of training data per task could be more sample-efficient than relying only on additional in-context demonstrations.

\subsection{Qualitative implications for POMDPs and APOMDPs}
\label{subsec:theory_pomdp_apomdp}

The linear MDP assumptions above are restrictive, but the decomposition in Proposition~\ref{prop:policy_gap} still gives useful qualitative guidance for the partially observed and ambiguous settings studied in our numerical experiments.

\paragraph{Partial observability.}
As discussed earlier, in POMDPs a sufficient statistic for decision-making is the belief state $b_t$. If we replace $\phi(s,a)$ with belief-action features $\phi(b,a)$, the same analysis applies formally only when the optimal belief-state Q-function is linear in those features and when the induced feature covariance is stable across support trajectories. As observation fidelity decreases, belief updates become noisier and the effective feature covariance can become more ill-conditioned, increasing $\kappa$ and the on-policy mismatch constant $C_{\mathrm{on}}$. Proposition~\ref{prop:policy_gap} then predicts larger optimality gaps and a larger benefit from additional training data.

\paragraph{Model ambiguity.}
In APOMDPs, the $\alpha$-MEU optimal Q-function need not be linear in a fixed feature map, and thus Proposition~\ref{prop:policy_gap} does not apply directly. Nevertheless, the same error decomposition suggests two relevant mechanisms. First, larger ambiguity sets can increase the complexity of the effective Q-function. Second, pessimistic or optimistic ambiguity criteria can change the induced visitation distribution, thereby changing the effective covariance and the on-policy mismatch constant. These effects are consistent with the finding of our numerical experiments that ambiguity size and ambiguity attitude jointly affect the optimality gap.

\section{Additional experimental results}\label{sec:appendix_additional_results}

\subsection{Comparison against DPT}\label{sec:appendix_dpt}

We construct the DPT training dataset using the trajectories $\{\mathcal{C}_{\tau}\}$ generated by Algorithm \ref{alg:construct_data}. Instead of using the serialized trajectories directly, we follow the approach in~\cite{DPT2023lee} to extract sequences of observations, actions, next observations, and rewards from the trajectories and organize them into a tabular format. Specifically, we construct the context based on $H_t$ and use the optimal action at time $t$ as the supervised label. We follow the same training procedure as described in~\cite{DPT2023lee} to train DPT. During evaluation, we use the same test tasks to roll out trajectories and compute the optimality gap. Figures~\ref{figure:mdp-pomdp-E3200} and~\ref{figure:apomdp-E3200} show the optimality gap across all three settings.

\paragraph{MDP.} In the MDP setting (Figure~\ref{figure:mdp-pomdp-E3200}, left), our approach achieves smaller optimality gaps than DPT across all planning horizons. The advantage is pronounced for longer horizons: at $T=10$, our approach reduces the gap from 26\% to 15\%, and at $T=15$ from 43\% to 15\%. At $T=5$, both approaches perform well, though our approach yields a smaller gap (3\% versus 7\%).

\paragraph{POMDP.} In the POMDP setting (Figure~\ref{figure:mdp-pomdp-E3200}, right), our approach outperforms DPT across all observation regimes. The gap is pronounced at $q=0.8$ and $q=1.0$, where our approach achieves optimality gaps of roughly 10\% and 9\% versus 25\% and 26\% for DPT, respectively. Even under the most challenging condition ($q=0.5$), our approach has a smaller gap (16\% versus 22\%).

\begin{figure}[thbp]
\centering
\begin{minipage}{0.49\textwidth}
\centering
\includegraphics[width=\linewidth]{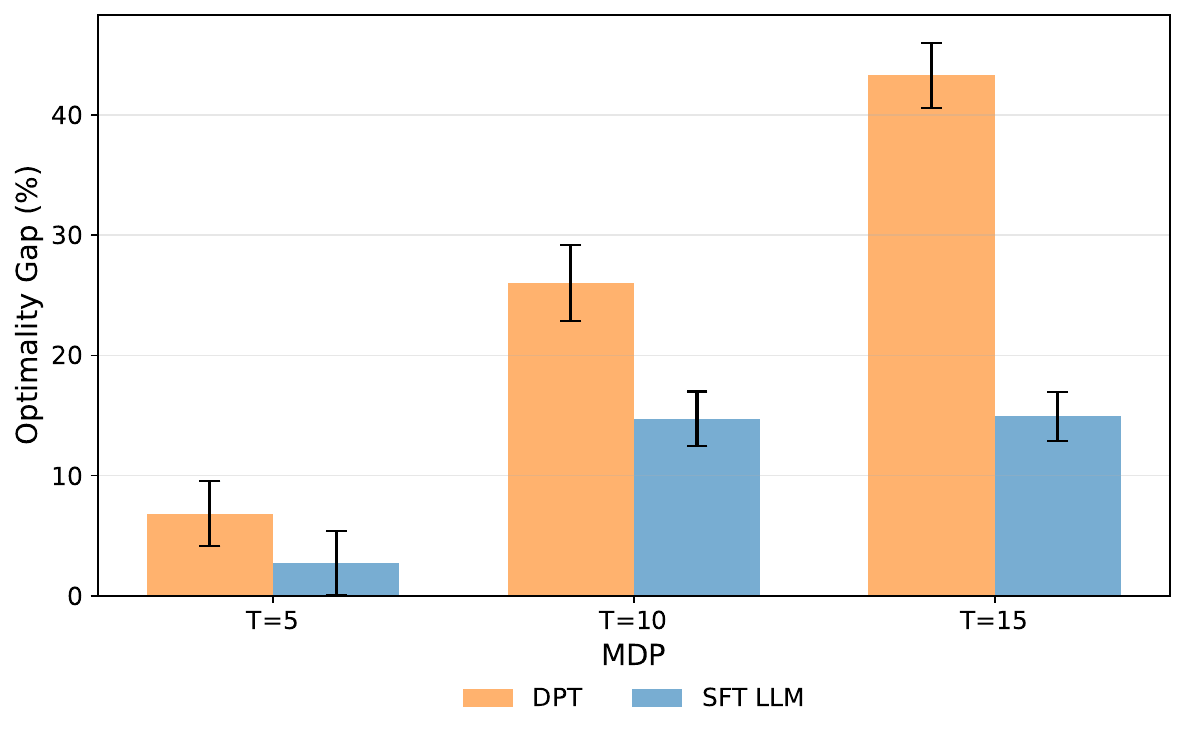}
\end{minipage}
\hfill
\begin{minipage}{0.49\textwidth}
\centering
\includegraphics[width=\linewidth]{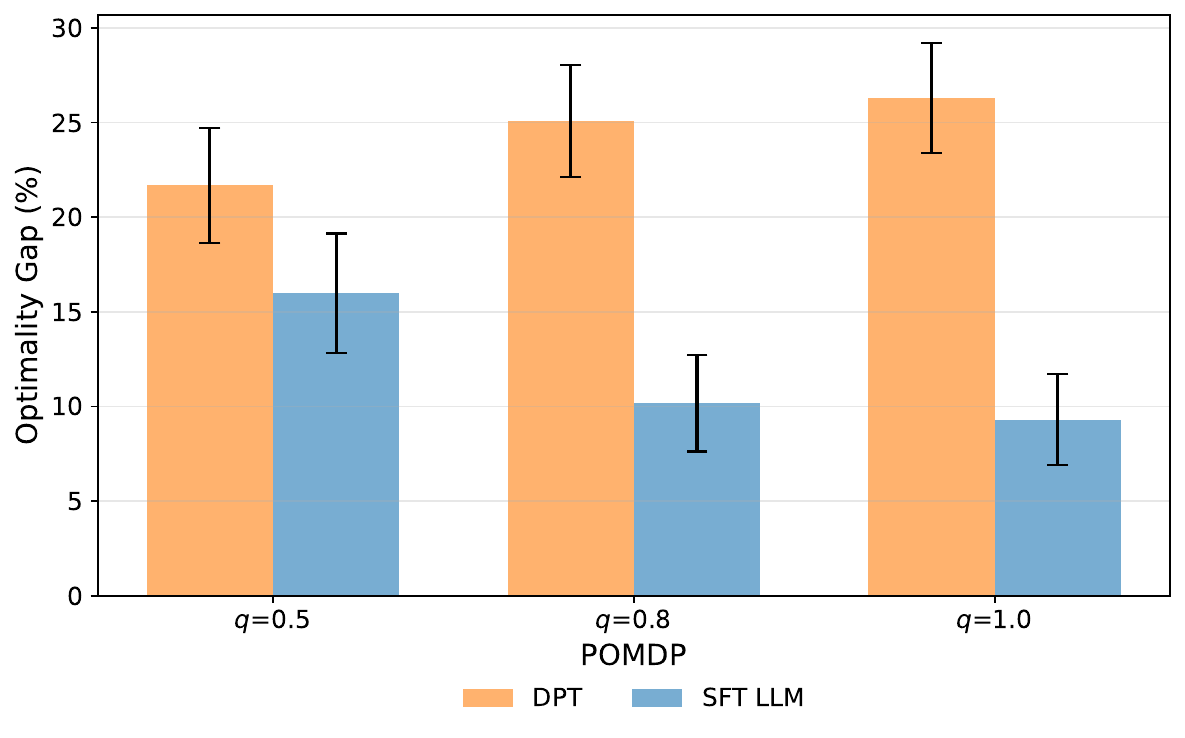}
\end{minipage}
\caption{Optimality gaps of our approach versus DPT with 3{,}200 training tasks for MDP (left) and POMDP (right). In the MDP panel, the three groups correspond to planning horizons $T \in \{5, 10, 15\}$. In the POMDP panel, $T=10$ and the three groups correspond to $q \in \{0.5, 0.8, 1.0\}$. Error bars represent 95\% confidence intervals.}
\label{figure:mdp-pomdp-E3200}
\end{figure}

\paragraph{APOMDP.} In the APOMDP setting (Figure~\ref{figure:apomdp-E3200}), our approach achieves smaller optimality gaps across all combinations of $\alpha$ and $|\mathcal{M}_\tau|$. Under $\alpha=0.5$, our approach maintains a consistent 10-percentage-point advantage. Under $\alpha=1.0$, the gaps for DPT grow faster (from 26\% to 37\%) than those for our approach (15\% to 22\%) when $|\mathcal{M}_\tau|$ increases.

\begin{figure}[thbp]
\centering
\begin{minipage}{0.32\textwidth}
\centering
\includegraphics[width=\linewidth]{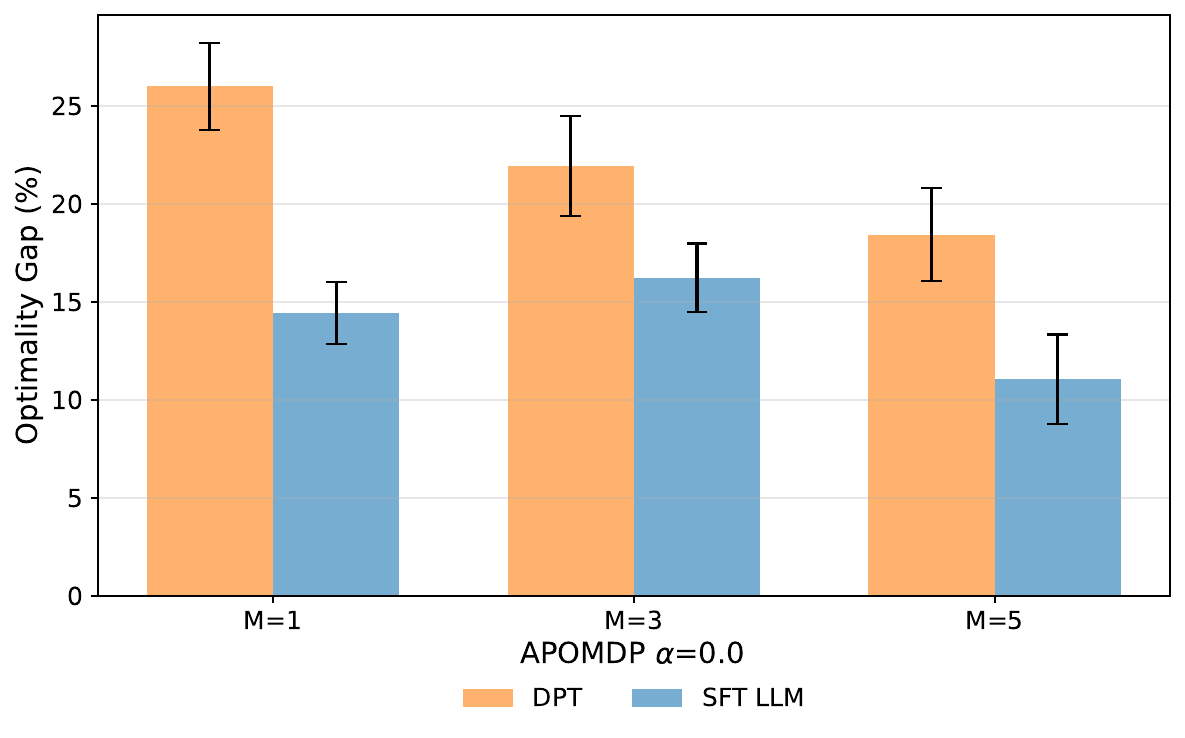}
\end{minipage}
\hfill
\begin{minipage}{0.32\textwidth}
\centering
\includegraphics[width=\linewidth]{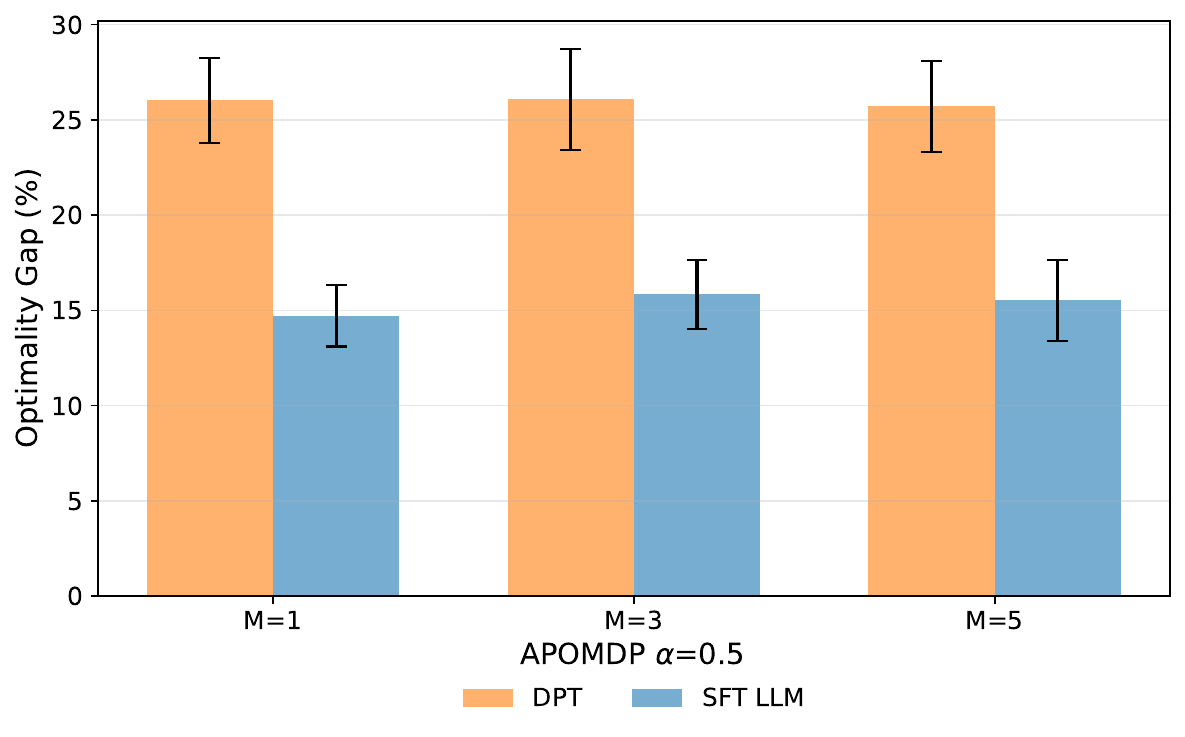}
\end{minipage}
\hfill
\begin{minipage}{0.32\textwidth}
\centering
\includegraphics[width=\linewidth]{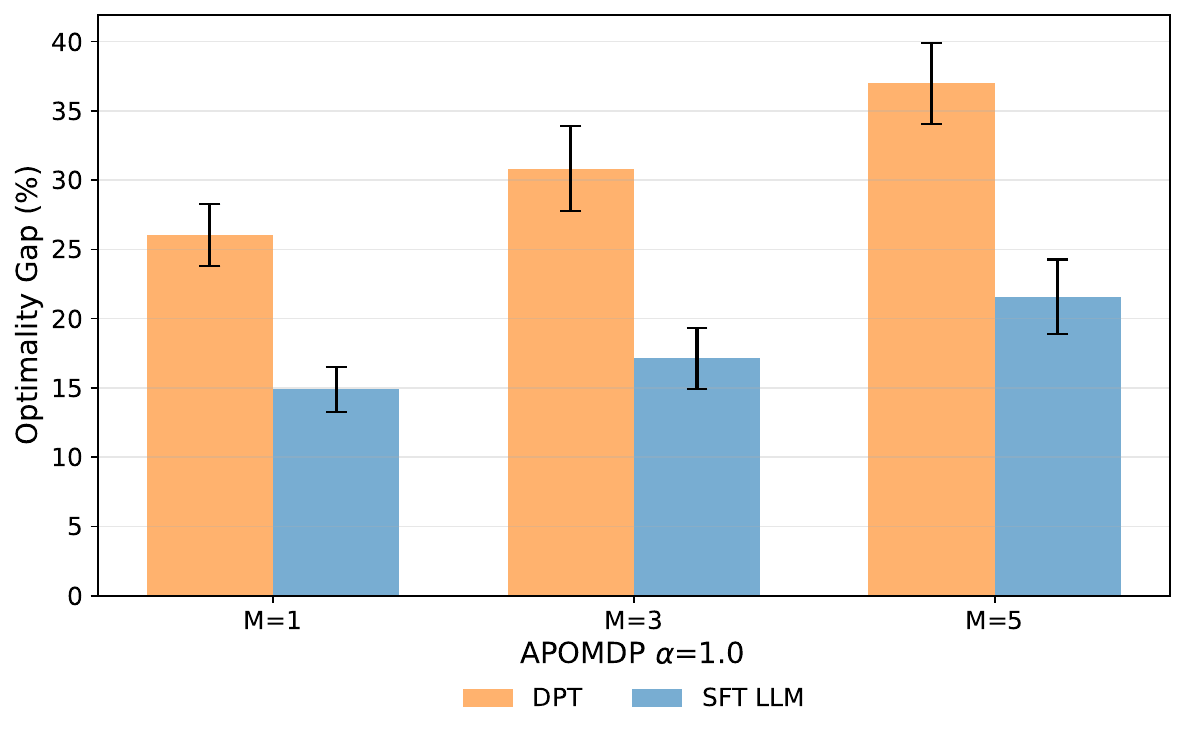}
\end{minipage}
\caption{Optimality gaps of our approach versus DPT with 3{,}200 training tasks for APOMDP. In all three panels, $T=10$ and $q=0.8$. The three panels correspond to $\alpha=0.0$ (left), $\alpha=0.5$ (center), and $\alpha=1.0$ (right). Within each panel, the three groups correspond to $|\mathcal{M}_\tau| \in \{1, 3, 5\}$. Error bars represent 95\% confidence intervals.}
\label{figure:apomdp-E3200}
\end{figure}

\subsection{Robustness to OOD test conditions}\label{sec:appendix_robustness}

We examine generalization beyond the training distribution with the model fixed after fine-tuning on 3{,}200 tasks. Figure~\ref{figure:robustness} shows optimality gaps for both in-distribution and OOD conditions.

\begin{figure}[thbp]
\centering
\begin{minipage}{0.32\textwidth}
\centering
\includegraphics[width=\linewidth]{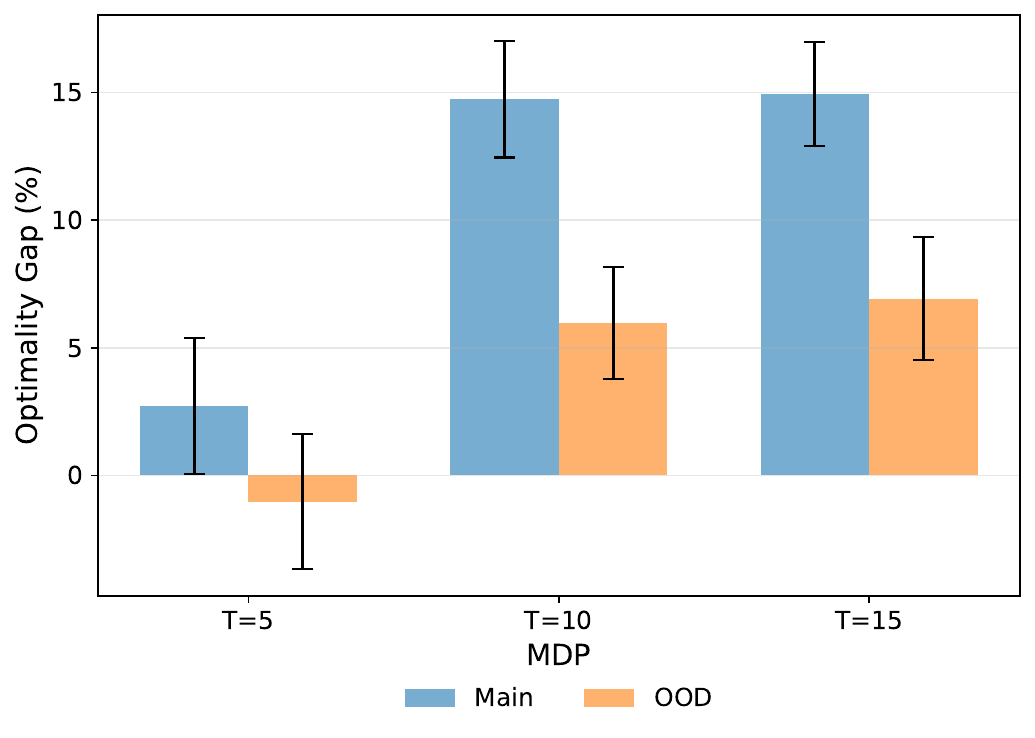}
\end{minipage}
\hfill
\begin{minipage}{0.32\textwidth}
\centering
\includegraphics[width=\linewidth]{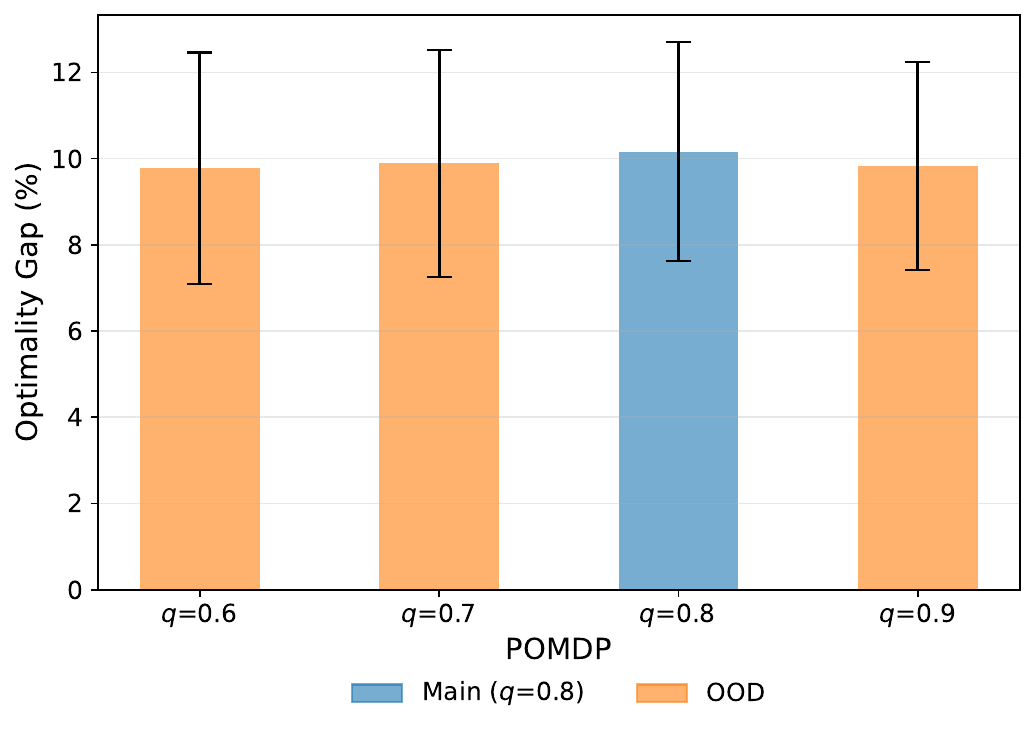}
\end{minipage}
\hfill
\begin{minipage}{0.32\textwidth}
\centering
\includegraphics[width=\linewidth]{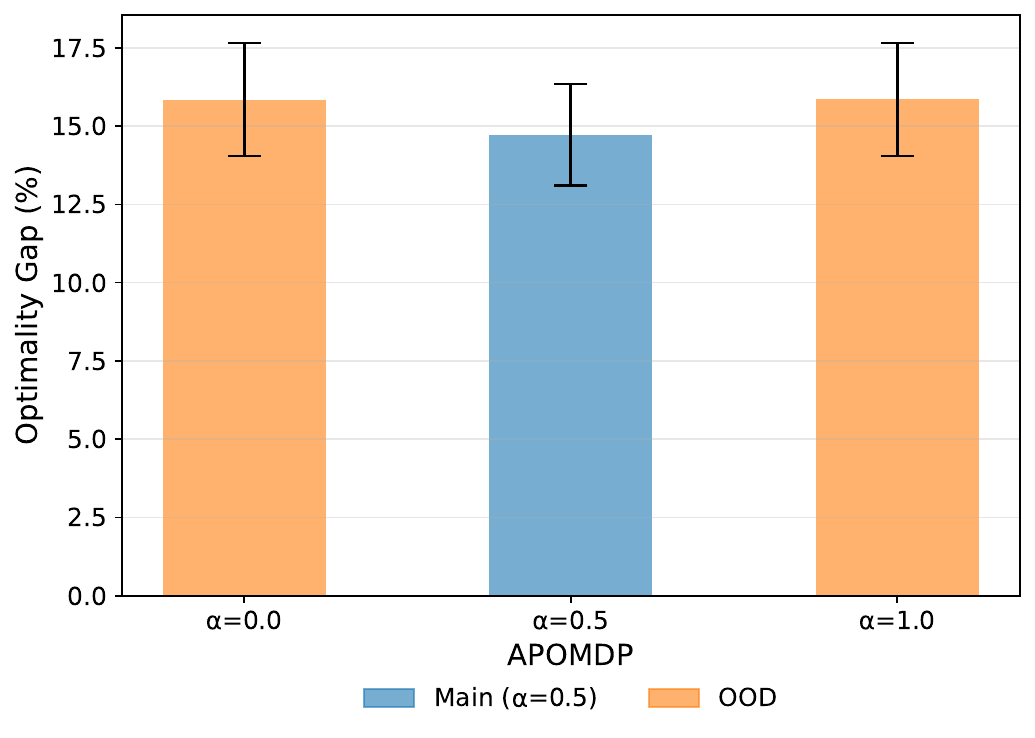}
\end{minipage}
\caption{Robustness of the fine-tuned LLM to the OOD test conditions with 3{,}200 training tasks for MDP (left), POMDP (center), and APOMDP (right). Blue bars: in-distribution. Orange bars: OOD. Error bars represent 95\% confidence intervals.}
\label{figure:robustness}
\end{figure}

\paragraph{MDP.} The model trained on $p \in [0.5, 1)$ is tested on $p \in [0.3, 0.5)$. The OOD gaps are small across all conditions: zero at $T=5$, 6\% at $T=10$, 7\% at $T=15$. The small OOD gaps may reflect that these test tasks are easier than the training tasks because smaller transition probabilities make state changes less frequent.

\paragraph{POMDP.} The model trained on $q=0.8$ is tested on $q \in \{0.6, 0.7, 0.9\}$. The OOD gaps are around 10\% in all conditions, suggesting the robustness to shifts in observation probability.

\paragraph{APOMDP.} The model trained on $\alpha=0.5$ is tested on $\alpha \in \{0.0, 1.0\}$. The OOD gaps are only marginally above the in-distribution gap.

\subsection{Quality of few-shot support trajectories}\label{sec:appendix_fewshot}

We examine whether demonstration quality affects performance by replacing rollouts under the optimal policy with rollouts under the random policy in the few-shot support trajectories. Figure~\ref{figure:fewshot} shows the optimality gaps for both conditions across all three planning horizons.

\begin{figure}[thbp]
\centering
\begin{minipage}{0.32\textwidth}
\centering
\includegraphics[width=\linewidth]{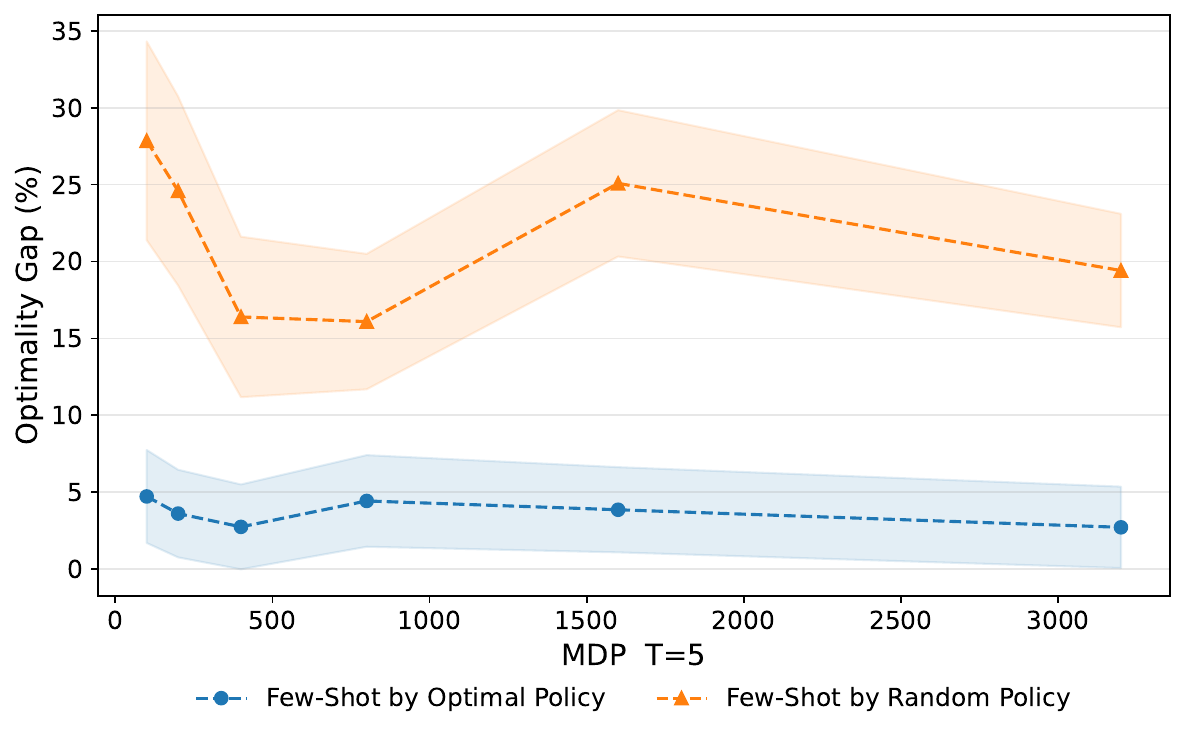}
\end{minipage}
\hfill
\begin{minipage}{0.32\textwidth}
\centering
\includegraphics[width=\linewidth]{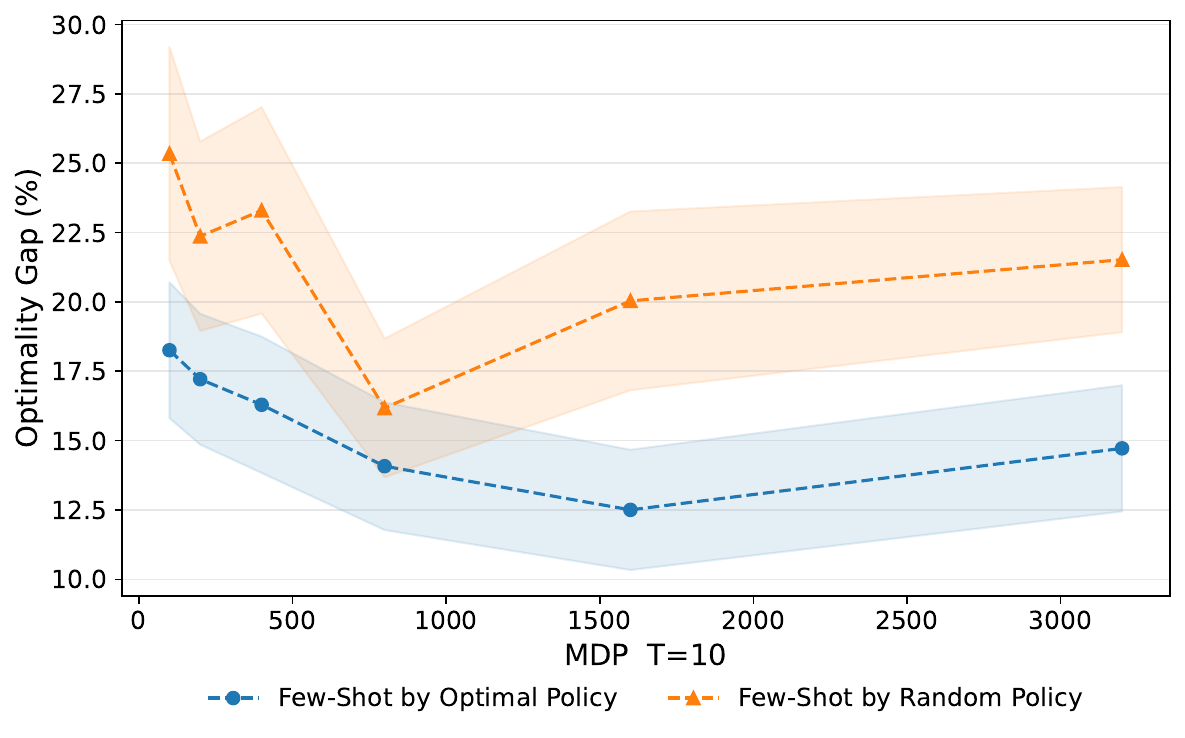}
\end{minipage}
\hfill
\begin{minipage}{0.32\textwidth}
\centering
\includegraphics[width=\linewidth]{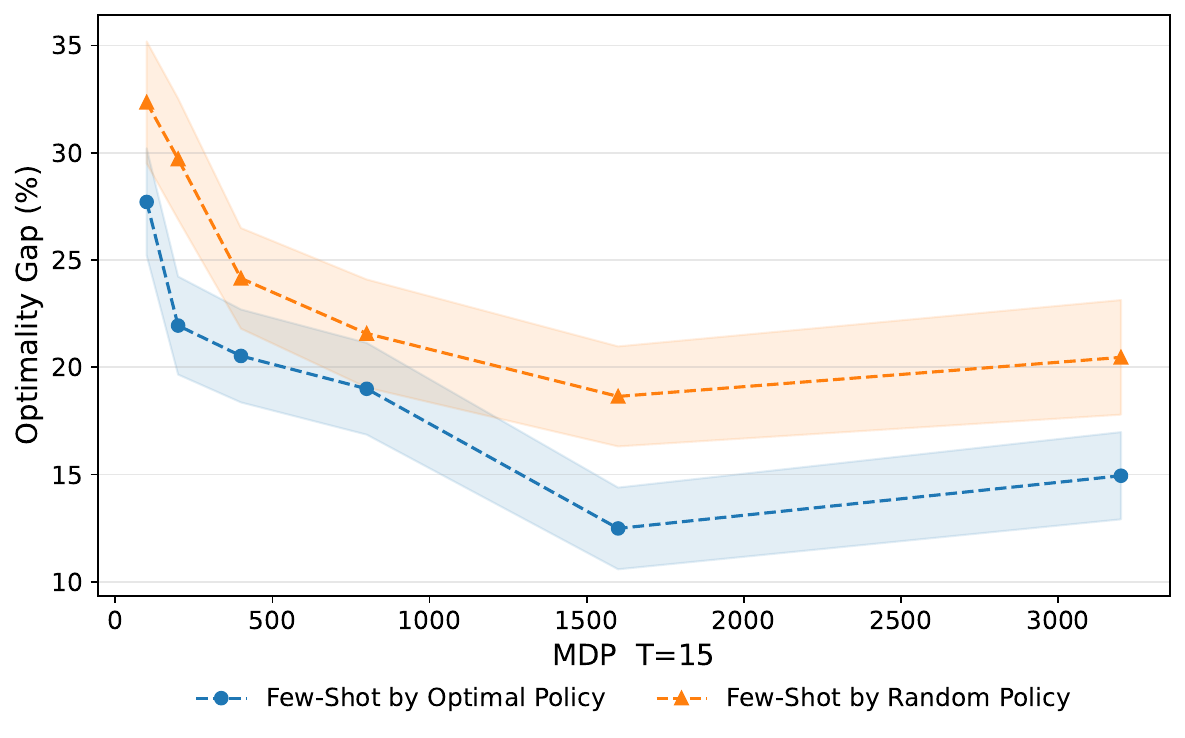}
\end{minipage}
\caption{Optimality gaps for few-shot support trajectories generated by the optimal policy (blue) and by the random policy (orange) under an MDP setting with $T=5$ (left), $T=10$ (center), and $T=15$ (right). The $x$-axis is the number of training tasks. The shaded areas represent the 95\% confidence interval.}
\label{figure:fewshot}
\end{figure}

The effect is pronounced at $T=5$. The gap is small (roughly 3\%--5\%) when the few-shot support trajectories are generated by the optimal policy, while it is around 20\% when they are generated by the random policy. At $T=10$ and $T=15$, we also see a gap of around 15\% for optimal support trajectories and 20\% for random support trajectories. These results suggest that demonstration quality could affect the performance of the fine-tuned LLM, and that the model can leverage high-quality demonstrations to achieve better performance. The gaps under a large number of training tasks suggest that the model may still benefit from high-quality support trajectories, even though we could provide more training tasks.

\subsection{Performance on Darkroom tasks}\label{sec:appendix_darkroom}

We apply our approach to the Darkroom task~\cite{DPT2023lee}, a $10 \times 10$ grid-world with hidden goals, sparse rewards (reward $+1$ at the goal and $0$ otherwise), and horizon $T=100$. We use 80 goal locations for fine-tuning and 20 for evaluation. Training trajectories are a mixture of 10 oracle and 100 random rollouts per goal. Test evaluations use five rollouts per goal, starting from $(0,0)$. The few-shot support trajectories are constructed from oracle rollouts for each test goal. The outcome of interest is the cumulative reward over the horizon. Figure~\ref{figure:darkroom} shows that the random policy achieves near-zero reward, while the oracle achieves a cumulative reward of approximately 92, and our fine-tuned LLM achieves approximately 87 (95\% of the oracle), suggesting the effectiveness of our approach in a challenging, partially observed environment with sparse rewards.

\begin{figure}[thbp]
\centering
\includegraphics[width=0.36\textwidth]{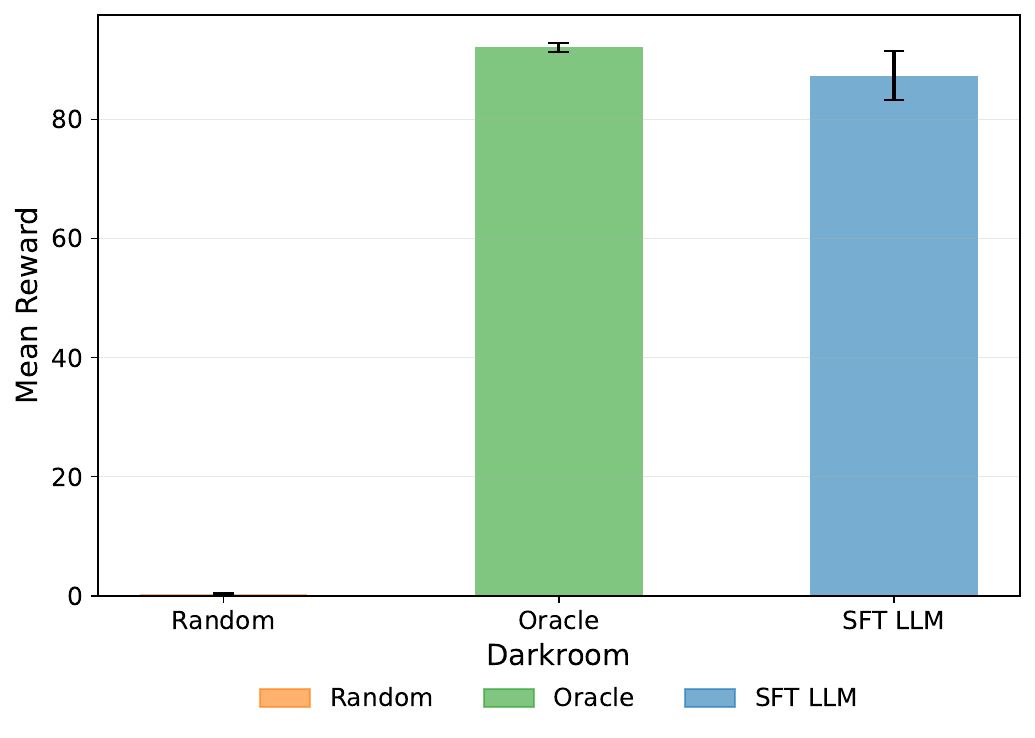}
\caption{Cumulative reward of the random policy, oracle, and our fine-tuned LLM for Darkroom tasks. Error bars represent 95\% confidence intervals.}
\label{figure:darkroom}
\end{figure}

\section{Limitations}\label{sec:appendix_limitations}

Our work has several limitations. First, supervised fine-tuning incurs nontrivial computational cost, which constrains the LLM sizes and the breadth of hyperparameters we are able to explore. Second, our framework relies on access to an optimal (or near-optimal) action oracle, which may be hard to compute once the problem becomes complex. Third, our theoretical analysis focuses on linear MDPs with a single-layer linear self-attention model; extensions to deeper architectures and to partially observed or model-ambiguous settings remain open (see Appendix~\ref{sec:appendix_theory_details} for partial qualitative results). Finally, our empirical evaluation is conducted on synthetic MDP, POMDP, and APOMDP settings, and validating the approach on real-world clinical datasets is an important next step.

\end{document}